\documentclass{article}

\usepackage{microtype}
\usepackage{graphicx}
\usepackage{subfigure}
\usepackage{booktabs} % for professional tables
\usepackage{tikz}
\usepackage{graphicx}
\usepackage[colorinlistoftodos]{todonotes}
\usepackage{natbib}

\usetikzlibrary{bayesnet}
\usepackage{scalerel}
\usepackage{amsmath,amssymb}

\usepackage{hyperref}
\usepackage{placeins}

\usepackage[arxiv]{icml2018}
\usepackage{amsthm}
\newtheorem{theorem}{Theorem}
\newtheorem{lemma}{Lemma}
\usepackage[makeroom]{cancel}
\usepackage{todonotes}
\usepackage{booktabs}
\newtheorem{proposition}{Proposition}

\newcommand\vfrac[2]{\ThisStyle{%
  \setbox0=\hbox{$\SavedStyle#1#2$}%
  \setbox2=\hbox{$\SavedStyle X$}%
  \ifdim\ht0>\ht2\setlength{\ht0}{\ht2}\fi%
  #1\mathord{\stretchto{\raisebox{2.3\LMpt}{$\SavedStyle/$}}{\ht0}}#2}}

\DeclareMathOperator{\x}{\mathbf{x}}
\DeclareMathOperator{\y}{\mathbf{y}}
\DeclareMathOperator{\z}{\mathbf{z}}
\DeclareMathOperator{\E}{\mathbb{E}}
\DeclareMathOperator{\D}{\mathcal{D}}
\DeclareMathOperator{\CCC}{\mathcal{C}}
\DeclareMathOperator{\SSS}{\mathcal{S}}
\DeclareMathOperator{\KL}{\D_{KL}}
\DeclareMathOperator{\R}{\mathbb{R}}

\makeatletter
\def\smallunderbrace#1{\mathop{\vtop{\m@th\ialign{##\crcr
   $\hfil\displaystyle{#1}\hfil$\crcr
   \noalign{\kern3\p@\nointerlineskip}%
   \tiny\upbracefill\crcr\noalign{\kern3\p@}}}}\limits}
\makeatother

\begin{document}

\twocolumn[
%\icmltitle{Neural Normalizing Flows}
\icmltitle{Neural Autoregressive Flows}

\icmlsetsymbol{equal}{*}

\begin{icmlauthorlist}
\icmlauthor{Chin-Wei Huang}{udem,elem,equal}
\icmlauthor{David Krueger}{udem,elem,equal}
\icmlauthor{Alexandre Lacoste}{elem}
\icmlauthor{Aaron Courville}{udem,cifar}
\end{icmlauthorlist}

\icmlaffiliation{udem}{MILA, University of Montreal}
\icmlaffiliation{elem}{Element AI}
\icmlaffiliation{cifar}{CIFAR fellow}
%\icmlaffiliation{to}{Department of Computation, University of Torontoland, Torontoland, Canada}
%\icmlaffiliation{goo}{Googol ShallowMind, New London, Michigan, USA}
%\icmlaffiliation{ed}{School of Computation, University of Edenborrow, Edenborrow, United Kingdom}

\icmlcorrespondingauthor{Chin-Wei Huang}{chin-wei.huang@umontreal.ca}
\icmlkeywords{Machine Learning, ICML}
\vskip 0.3in
]

\printAffiliationsAndNotice{\icmlEqualContribution} % otherwise use the standard text.

%%%%%%%%%%%%%%%%%%%%%%%%%%%%%%%%%%%%%%%%%%%%%%%%%%%%%%%%%%%%%%%%%
%%%%%%%%%%%%%%%%%%%%%%%%%%%%%%%%%%%%%%%%%%%%%%%%%%%%%%%%%%%%%%%%%
\begin{abstract}
Normalizing flows and autoregressive models have been successfully combined to produce state-of-the-art results in density estimation, via Masked Autoregressive Flows (MAF) \citep{papamakarios2017masked}, and to accelerate state-of-the-art WaveNet-based speech synthesis to 20x faster than real-time \citep{oord2017parallel}, via Inverse Autoregressive Flows (IAF) \citep{kingma2016improved}.
We unify and generalize these approaches, replacing the (conditionally) affine univariate transformations of MAF/IAF with a more general class of invertible univariate transformations expressed as monotonic neural networks.
We demonstrate that the proposed {\bf neural autoregressive flows (NAF)} are universal approximators for continuous probability distributions,
and their greater expressivity allows them to better capture multimodal target distributions.
Experimentally, NAF yields state-of-the-art performance on a suite of density estimation tasks and outperforms IAF in variational autoencoders trained on binarized MNIST. \footnote{Implementation can be found at \href{https://github.com/CW-Huang/NAF/}{https://github.com/CW-Huang/NAF/}}
\end{abstract}

% TODO: NAF --> NNF????? (since AF --> IAF?)

%%%%%%%%%%%%%%%%%%%%%%%%%%%%%%%%%%%%%%%%%%%%%%%%%%%%%%%%%%%%%%%%%
%%%%%%%%%%%%%%%%%%%%%%%%%%%%%%%%%%%%%%%%%%%%%%%%%%%%%%%%%%%%%%%%%
\section{Introduction}
Invertible transformations with a tractable Jacobian, also known as {\bf normalizing flows}, are useful tools in many machine learning problems, for example:
(1) In the context of \textbf{deep generative models}, training necessitates evaluating data samples under the model's inverse transformation \cite{dinh2016density}.
Tractable density is an appealing property for these models, since it allows the objective of interest to be directly optimized; whereas other mainstream methods rely on alternative losses, in the case of intractable density models \cite{kingma2013auto,rezende2014stochastic}, or implicit losses, in the case of adversarial models \cite{goodfellow2014generative}. 
(2) In the context of \textbf{variational inference} \cite{rezende2015variational}, they can be used to improve the variational approximation to the posterior by parameterizing more complex distributions. 
This is important since a poor variational approximation to the posterior can fail to reflect the right amount of \textit{uncertainty},
and/or be biased \cite{turner2011two}, resulting in inaccurate and unreliable predictions. 
We are thus interested in improving techniques for normalizing flows.

Recent work by \citet{kingma2016improved} reinterprets autoregressive models as invertible transformations suitable for constructing normalizing flows. 
The inverse transformation process, unlike sampling from the autoregressive model, is not sequential and thus can be accelerated via parallel computation. 
This allows multiple layers of transformations to be stacked, increasing expressiveness for better variational inference \cite{kingma2016improved} or better density estimation for generative models \cite{papamakarios2017masked}. 
Stacking also makes it possible to improve on the sequential conditional factorization assumed by autoregressive models such as PixelRNN or PixelCNN \cite{oord2016pixel}, and thus define a more flexible joint probability.

We note that the normalizing flow introduced by \citet{kingma2016improved} only applies an affine transformation of each scalar random variable.
Although this transformation is conditioned on preceding variables, the resulting flow can still be susceptible to bad local minima, and thus failure to capture the multimodal shape of a target density; see Figure \ref{fig:4mode_iaf} and \ref{fig:4mode_maf}.  

\subsection{Contributions of this work}
We propose replacing the conditional affine transformation of \citet{kingma2016improved} with a more rich family of transformations, and note the requirements for doing so.
We determine that very general transformations, for instance parametrized by deep neural networks, are possible.
We then propose and evaluate several specific monotonic neural network architectures which are more suited for learning multimodal distributions.
Concretely, our method amounts to using an autoregressive model to output the weights of multiple independent transformer networks, each of which operates on a single random variable, replacing the affine transformations of previous works. %\todo{Aaron: this part about "given the 

Empirically, we show that our method works better than the state-of-the-art affine autoregressive flows of \citet{kingma2016improved} and \citet{papamakarios2017masked}, 
both as a sample generator which captures multimodal target densities with higher fidelity, and as a density model which more accurately evaluates the likelihood of data samples drawn from an unknown distribution.

We also demonstrate that our method is a universal approximator on proper distributions in real space, 
which guarantees the expressiveness of the chosen parameterization and supports our empirical findings.

\begin{figure}
\centering
 \includegraphics[width=0.45\textwidth]{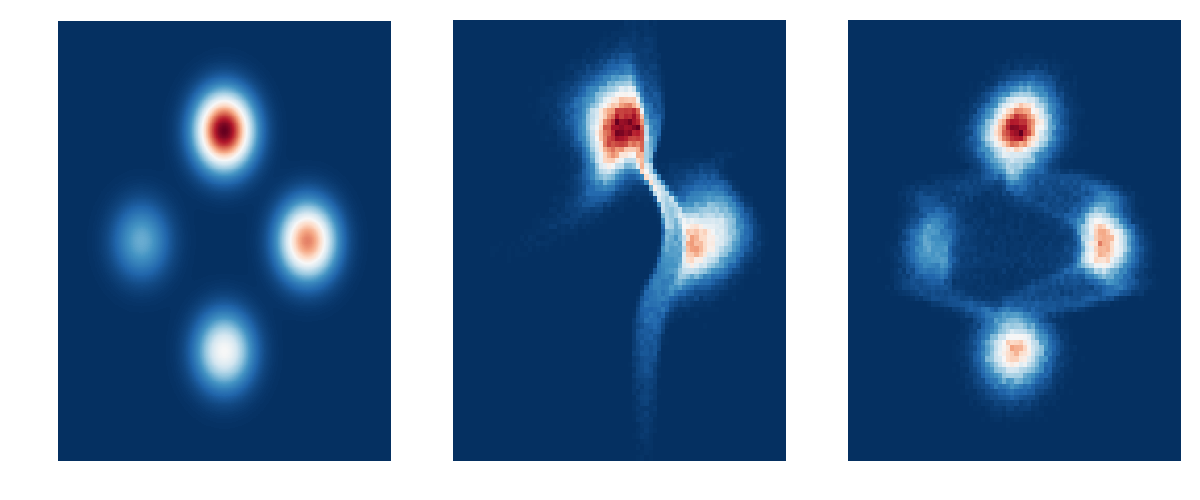}
\caption{Energy function fitting using IAF. \\ Left: true distribution. Center: IAF-affine. Right: IAF-DSF.}
\label{fig:4mode_iaf}
\end{figure}
\begin{figure}
\centering
 \includegraphics[width=0.45\textwidth]{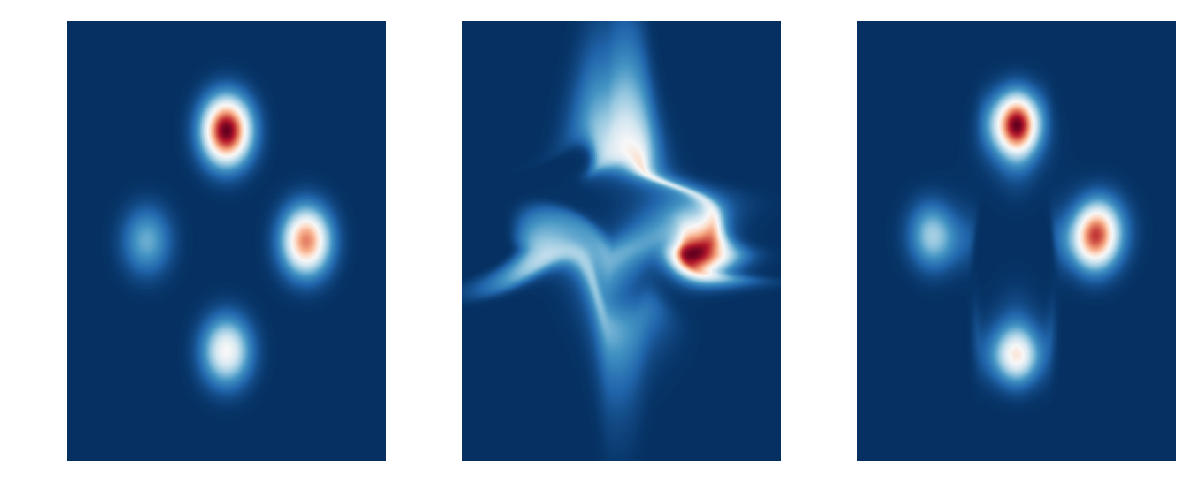}
\caption{Density estimation using MAF. \\ Left: true distribution. Center: MAF-affine. Right: MAF-DSF.}
\label{fig:4mode_maf}
\end{figure}

%%%%%%%%%%%%%%%%%%%%%%%%%%%%%%%%%%%%%%%%%%%%%%%%%%%%%%%%%%%%%%%%%
%%%%%%%%%%%%%%%%%%%%%%%%%%%%%%%%%%%%%%%%%%%%%%%%%%%%%%%%%%%%%%%%%
\section{Background}
\label{sec:background}
% 1. What is an NF?
A (finite) {\bf normalizing flow (NF)}, or {\bf flow}, is an invertible function $f_\theta: \mathcal{X} \rightarrow \mathcal{Y}$ used to express a transformation between random variables \footnote{
We use $\x$ and $\y$ to denote inputs and outputs of a function, {\it not} the inputs and targets of a supervised learning problem.
}.
Since $f$ is invertible, the change of variables formula can be used to translate between densities $p_Y(\y)$ and $p_X(\x)$:
\begin{align}
p_Y(\y) = \left| \frac{\partial f(\x)}{\partial \x} \right|^{-1} p_X(\x)
\label{eq:changeovar}
\end{align}	
The determinant of $f$'s Jacobian appears on the right hand side to account for the way in which $f$ can (locally) expand or contract regions of $X$, thereby lowering or raising the resulting density in those regions' images in $Y$.
Since the composition of invertible functions is itself invertible, complex NFs are often formed via function composition (or ``stacking'') of simpler NFs.

% 2. What do we do with them? How do we train them?
%	TODO: emphasize that there are other applications!? (or do that in the intro...)
Normalizing flows are most commonly trained to produce an output distribution $p_{Y}(\y)$ which matches a target distribution (or, more generally, energy function) $p_{\mathrm{target}}(\y)$ as measured by the KL-divergence $KL(p_Y(\y) || p_{\mathrm{target}}(\y))$.
When $X$ or $Y$ is distributed by some simple distribution, such as uniform or standard normal, we call it an unstructured noise; and we call it a structured noise when the distribution is complex and correlated.
Two common settings are maximum likelihood and variational inference.
Note that these two settings are typically viewed as optimizing different directions of the KL-divergence, whereas we provide a unified view in terms of different input and target distributions. A detailed derivation is presented in the appendix (See Section~\ref{sec:KLnMLE}).

For maximum likelihood applications \citep{dinh2016density,papamakarios2017masked}, $p_{\mathrm{target}}(\y)$ is typically a simple prior over latent variable $\y$, and $f$ attempts to disentangle the complex empirical distribution of the data, $p_X(\x)$ into a simple latent representation $p_Y(\y)$ matching the prior (\textit{ structured to unstructured}) \footnote{
It may also be possible to form a generative model from such a flow, by passing samples from the prior $p_{\mathrm{target}}(\y)$ through $f^{-1}$, although the cost of doing so may vary.
For example, RealNVP \citep{dinh2016density} was devised as a generative model, and its inverse computation is as cheap as its forward computation, whereas MAF \citep{papamakarios2017masked} is designed for density estimation and is much more expensive to sample from.
For the NAF architectures we employ, we do not have an analytic expression for $f^{-1}$, but it is possible to approximate it numerically.}.

In a typical application of variational inference \citep{rezende2015variational,kingma2016improved}, $p_{\mathrm{target}}(\y)$ is a complex posterior over latent variables $\y$, and $f$ transforms a simple input distribution (for instance a standard normal distribution) over $\x$ into a complex approximate posterior $p_Y(\y)$ ({\textit unstructured to structured}). 
In either case, since $p_X$ does not depend on $\theta$, the gradients of the KL-divergence are typically estimated by Monte Carlo:
\begin{align}
\nabla_\theta \mathcal{D}_{KL}\big(p_Y(\y) &|| p_{\mathrm{target}}(\y)\big) \nonumber\\
&= \nabla_\theta \int_\mathcal{Y} p_Y(\y) \log\frac{p_Y(\y)}{p_{\mathrm{target}}(\y)}  \mathrm{d}\y \nonumber \\
&= \int_\mathcal{X} p_X(\x) \nabla_\theta \log\frac{p_Y(\y)}{p_{\mathrm{target}}(\y)}  \mathrm{d}\x 
\end{align}
Applying the change of variables formula from Equation 1 to the right hand side of Equation 2 yields:
\begin{align}
\E_{
\begin{subarray}{l}
\x \sim p_X(\x) \\
\y = f_\theta(\x)
\end{subarray}
}
\left[ \nabla_\theta \log \left| \frac{\partial f_\theta(\x)}{\partial \x} \right|^{-1} p_X(\x) - \nabla_\theta \log {p_{\mathrm{target}}(\y)} \right]
%\label{eqn:}
\end{align}
Thus for efficient training, the following operations must be tractable and cheap:
\begin{enumerate}
\item Sampling $\x \sim p_X(\x)$
\item Computing $\y = f(\x)$
\item Computing the gradient of the log-likelihood of $\y = f(\x);\; \x \sim p_X(\x)$ under both $p_Y(\y)$ and $p_\mathrm{target}(\y)$
\item Computing the gradient of the log-determinant of the Jacobian of $f$
\end{enumerate}
Research on constructing NFs, such as our work, focuses on finding ways to parametrize flows which meet the above requirements while being maximally flexible in terms of the transformations which they can represent.
Note that some of the terms of of Equation 3 may be constant with respect to $\theta$ \footnote{ There might be some other parameters other than $\theta$ that are learnable, such as parameters of $p_X$ and $p_{target}$ in the variational inference and maximum likelihood settings, respectively.} and thus trivial to differentiate, such as $p_X(\x)$ in the maximum likelihood setting.

% MAF/IAF (COMPARISON)
\begin{figure}
\centering
 \includegraphics[width=0.47\textwidth]{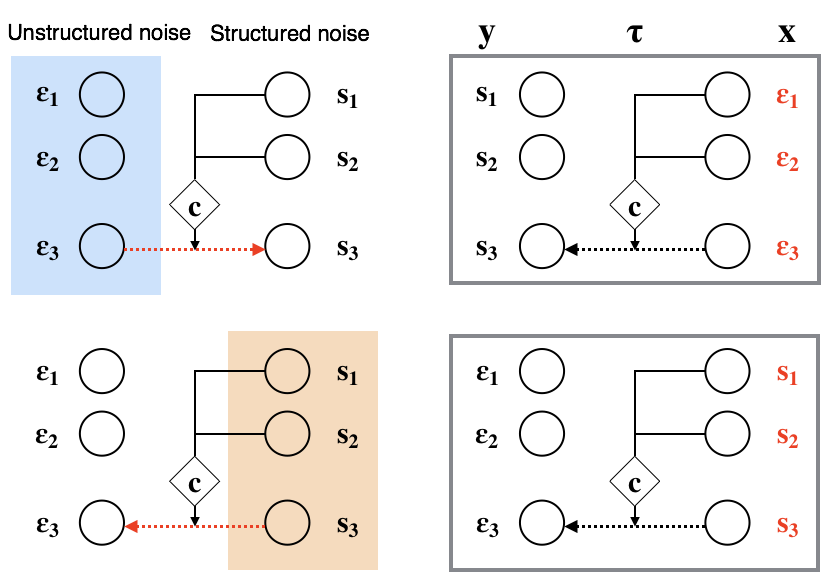}
\caption{
Difference between autoregressive and inverse autoregressive transformations (left), and between IAF and MAF (right).
{\bf Upper left}: sample generation of an autoregressive model. Unstructured noise is transformed into structured noise.
{\bf Lower left}: inverse autoregressive transformation of structured data.  Structured variables are transformed into unstructured variables.
{\bf Upper right}: IAF-style sampling. 
{\bf Lower right}: MAF-style evaluation of structured data. 
$\epsilon$ represents unstructured noise and $s$ represents structured noise. 
}
\end{figure}

% 3. AAFs
{\bf Affine autoregressive flows (AAFs)} \footnote{
Our terminology differs from previous works, and hence holds the potential for confusion, but we believe it is apt.
Under our unifying perspective, NAF, IAF, AF, and MAF all make use of the same principle, which is an invertible transformer conditioned on the outputs of an autoregressive (and emphatically {\it not} an {\it inverse} autoregressive) conditioner.
}, 
such as inverse autoregressive flows (IAF) \citep{kingma2016improved}, are one particularly successful pre-existing approach.
Affine autoregressive flows yield a triangular Jacobian matrix, so that the log-determinant can be computed in linear time, as the sum of the diagonal entries on log scale.
In AAFs, the components of $\x$ and $\y$ are given an order (which may be chosen arbitrarily), and $y_t$ is computed as a function of $x_{1:t}$.
% 4. decomposing AAFs into c, tau
Specifically, this function can be decomposed via an autoregressive {\bf conditioner}, $c$, and an invertible {\bf transformer}, $\tau$, as \footnotemark:
\footnotetext{ \citet{Dinh2014} use $m$ and $g^{-1}$ to denote $c$ and $\tau$, and refer to them as the ``coupling function'' and ``coupling law'', respectively.}
\begin{align}
y_t \doteq f(x_{1:t}) = \tau(c(x_{1:t-1}), x_t)
\label{eq:f}
\end{align}
It is possible to efficiently compute the output of $c$ for all $t$ in a single forward pass using a model such as {\bf MADE} \citep{MADE}, as pointed out by \citet{kingma2016improved}.

% 5. affine... or not!
In previous work, 
$\tau$ is taken to be an affine transformation with parameters $\mu \in \R, \sigma > 0$ output from $c$. For instance 
\citet{dinh2016density} use:
\begin{align}
\tau(\mu, \sigma, x_t) = \mu + \sigma x_t
\end{align}
with $\sigma$ produced by an exponential nonlinearity. \citet{kingma2016improved} use:
\begin{align}
\tau(\mu, \sigma, x_t) = \sigma x_t + (1-\sigma) \mu
\end{align}
with $\sigma$ produced by a sigmoid nonlinearity.
Such transformers are trivially invertible, but their relative simplicity also means that the expressivity of $f$ comes entirely from the complexity of $c$ and from stacking multiple AAFs (potentially using different orderings of the variables) \footnote{
Permuting the order of variables is itself a normalizing flow that does not expand or contract space and can be inverted by another permutation.
}.
However, the only requirements on $\tau$ are:
\begin{enumerate}
\item The transformer $\tau$ must be invertible as a function of $x_t$.
\item $\frac{d y_t}{d x_{t}}$ must be cheap to compute. 
\end{enumerate} 
This raises the possibility of using a more powerful transformer in order to increase the expressivity of the flow.

%%%%%%%%%%%%%%%%%%%%%%%%%%%%%%%%%%%%%%%%%%%%%%%%%%%%%%%%%%%%%%%%%
%%%%%%%%%%%%%%%%%%%%%%%%%%%%%%%%%%%%%%%%%%%%%%%%%%%%%%%%%%%%%%%%%
\section{Neural Autoregressive Flows}

% ARCHITECTURES
\begin{figure}[th]
\centering
\subfigure[Neural autoregressive flows (NAF)]{
\centering
\includegraphics[width=0.46\textwidth]{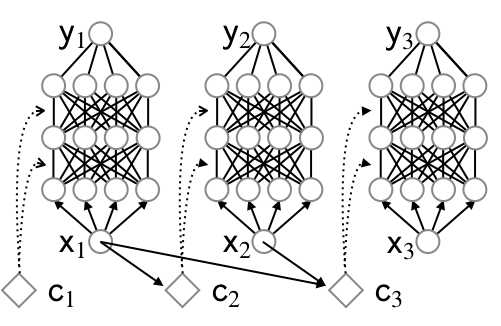}
\label{fig:graph_naf}}\\
\subfigure[DSF]{
\centering
\includegraphics[width=0.2\textwidth]{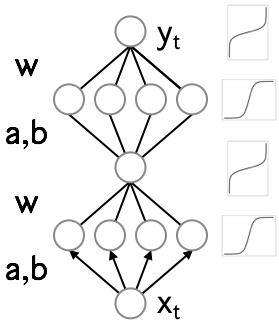}
\label{fig:graph_dsf}}
\subfigure[DDSF]{
\centering
\includegraphics[width=0.2\textwidth]{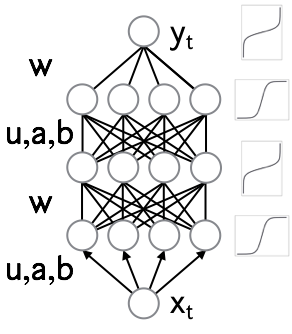}
\label{fig:graph_ddsf}}
\caption{
{\bf Top:} In neural autoregressive flows, the transformation of the current input variable is performed by an MLP whose parameters are output from 
an autoregressive conditioner model, $c_t \doteq c(x_{1:t-1})$, which incorporates information from previous input variables. %($c_t \doteq c(x_{1:t-1})$) is used to condition the transformation of the current input variable, which by specifying the weights of an MLP. 
\hspace{1mm}
{\bf Bottom:} The architectures we use in this work: deep sigmoidal flows (DSF) and deep dense sigmoidal flows (DDSF).
See section \ref{subsec:arch} for details.
}
\label{fig:arch}
\end{figure}

% HISTOGRAMS
\begin{figure}
 \centering
 \includegraphics[width=0.9\linewidth]{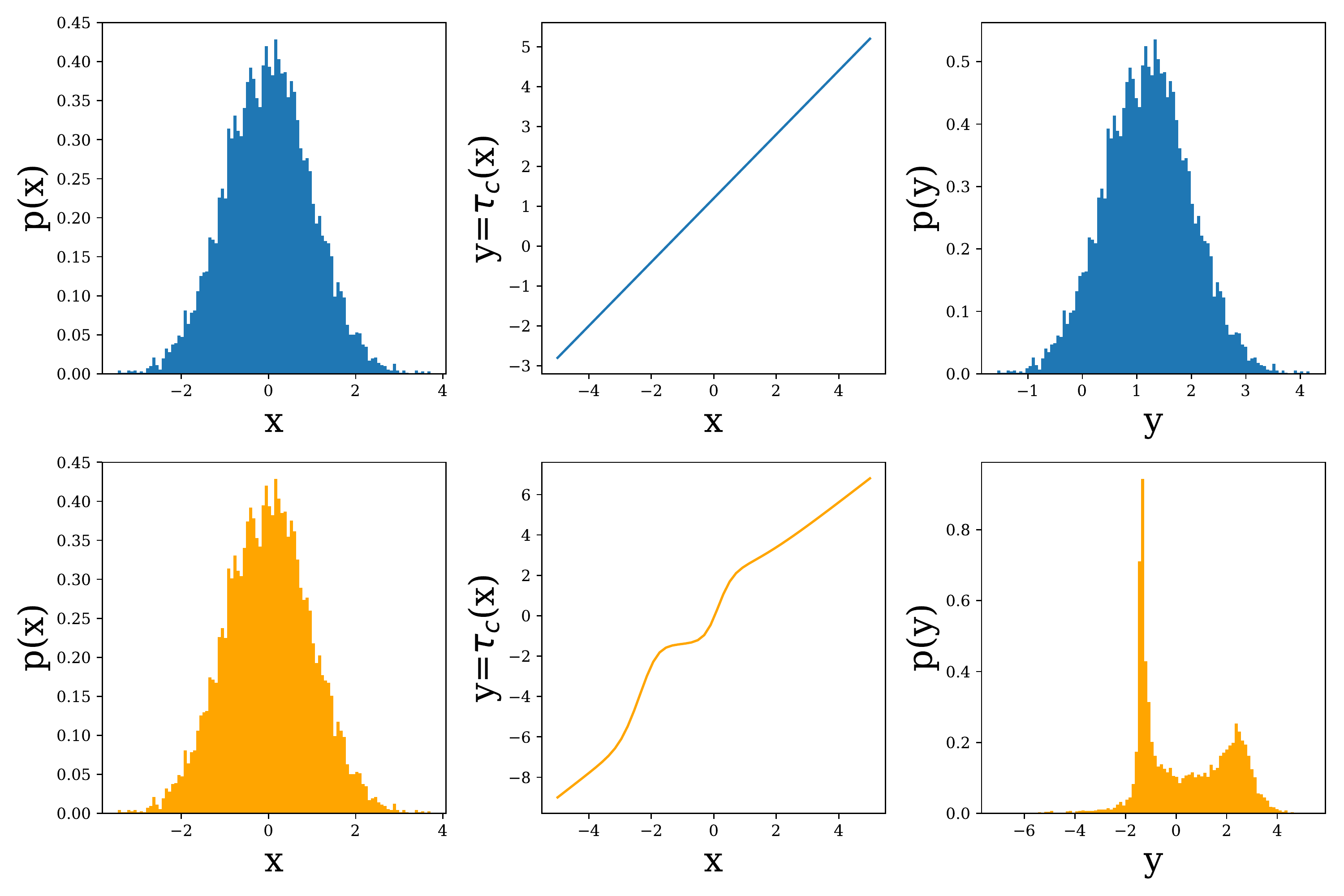}%
 \caption{{\small
	Illustration of the effects of traditional IAF (top), and our proposed NAF (bottom).
	Areas where the slope of the transformer $\tau_c$ is greater/less than 1, are compressed/expanded (respectively) in the output distribution.
	Inflection points in $\tau_c(x_t)$ (middle) can transform a unimodal $p(x_t)$ (left) into a multimodal $p(y_t)$ (right); NAF allows for such inflection points, whereas IAF does not.
}}
  \label{fig:illustrate}
\end{figure}

% 1. NAFs
We propose replacing the affine transformer used in previous works with a neural network, yielding a more rich family of distributions with only a minor increase in computation and memory requirements.
Specifically, 
\begin{align}
\tau(c(x_{1:t-1}),x_t) = \mathrm{DNN}(x_t;\; \phi&=c(x_{1:t-1}))
\end{align}
is a deep neural network which takes the scalar $x_t$ as input and produces $y_t$ as output, and its weights and biases are given by the outputs of $c(x_{1:t-1}) $\footnote{ We'll sometimes write $\tau_c$ for $\tau(c(x_{1:t-1}), \cdot)$.} (see Figure \ref{fig:graph_naf}).
We refer to these values $\phi$ as {\bf pseudo-parameters}, in order to distinguish them from the statistical parameters of the model.

% 2. monotonic DNNs
We now state the condition for NAF to be strictly monotonic, and thus invertible (as per requirement 1):
\begin{proposition}
Using strictly positive weights and strictly monotonic activation functions for $\tau_c$ is sufficient for the entire network to be strictly monotonic. 
\label{proposition:mono}
\end{proposition}
Meanwhile, $\frac{d y_t}{d x_t}$ and gradients wrt the pseudo-parameters \footnote{
Gradients for pseudo-parameters are backpropagated through the conditioner, $c$, in order to train its parameters.
} 
can all be computed efficiently via backpropagation (as per requirement 2).

% 3. inducing multimodality
Whereas affine transformers require information about multimodality in $y_t$ to flow through $x_{1:t-1}$, our {\bf neural autoregressive flows (NAFs)} are able to induce multimodality more naturally, via inflection points in $\tau_c$, as shown in Figure~\ref{fig:illustrate}. 
Intuitively, $\tau_c$ can be viewed as analogous to a cumulative distribution function (CDF), so that its derivative corresponds to a PDF, where its inflection points yield local maxima or minima.

%%%%%%%%%%%%%%%%%%%%%%%%%%%%%%%%%%%%%%%%%
\subsection{Transformer Architectures} %{Architectural Details}
\label{subsec:arch}

% Using sigmoids
In this work, we use two specific architectures for $\tau_c$, which we refer to as 
{\bf deep sigmoidal flows (DSF)} and {\bf deep dense sigmoidal flows (DDSF)} (see Figure~\ref{fig:graph_dsf},~\ref{fig:graph_ddsf} for an illustration).
We find that small neural network transformers of 1 or 2 hidden layers with 8 or 16 sigmoid units perform well across our experiments, although there are other possibilities worth exploring (see Section~\ref{subsec:alternatives}).
Sigmoids contain inflection points, and so can easily induce inflection points in $\tau_c$, and thus multimodality in $p(y_t)$.
We begin by describing the DSF transformation, which is already sufficiently expressive to form a universal approximator for probability distributions, as 
we prove in section \ref{sec:universal}.

% DSF and need for logits
The DSF transformation resembles an MLP with a single hidden layer of sigmoid units.
Naive use of sigmoid activation functions would restrict the range of $\tau_c$, however, and result in a model that assigns 0 density to sufficiently large or small $y_t$, which is problematic when $y_t$ can take on arbitrary real values.
We address this issue by applying the inverse sigmoid (or ``logit'') function at the output layer.
To ensure that the output's preactivation is in the domain of the logit (that is, $(0,1)$), we combine the output of the sigmoid units via an attention-like \citep{bahdanau2014neural} softmax-weighted sums:
\begin{equation}
y_t = \sigma^{-1}(\smallunderbrace{w^T}_{1\times d}\cdot\sigma(\smallunderbrace{a}_{d\times1} \cdot \smallunderbrace{x_t}_{1\times1}  +  {\smallunderbrace{b}_{d\times1}})))
\label{eq:sf}
%\label{eq:dsf}
\end{equation}
where $0<w_{i,j}<1$, $\sum_i w_{i,j}=1$, $a_{s,t}>0$, and $d$ denotes the number of hidden units \footnote{
Constraints on the variables are enforced via activation functions; $w$ and $a$ are outputs of a $\mathrm{softmax}$, and $\mathrm{softplus}$ or $\mathrm{exp}$, respectively.
}.

Since all of the sigmoid activations are bounded between 0 and 1, the final preactivation (which is their convex combination) is as well.
The complete DSF transformation can be seen as mapping the original random variable to a different space through an activation function, where doing affine/linear operations is non-linear with respect to the variable in the original space, and then mapping it back to the original space through the inverse activation.

When stacking multiple sigmoidal transformation, we realize it resembles an MLP with bottleneck as shown by the bottom left of Figure \ref{fig:arch}. 
A more general alternative is the {\bf deep dense sigmoidal flow (DDSF)}, which takes the form of a fully connected MLP:
\begin{align}
&h^{(l+1)} = \nonumber\\
&\sigma^{-1}(\smallunderbrace{w^{(l+1)}}_{d_{l+1}\times d_{l+1}} \cdot\, \sigma( \smallunderbrace{a^{(l+1)}}_{d_{l+1}} \odot \smallunderbrace{u^{(l+1)}}_{d_{l+1}\times d_{l}} \cdot \smallunderbrace{h^{(l)}}_{d_{l}} + {\smallunderbrace{b^{(1+1)}}_{d_{l+1}}} ))
\label{eq:ddsf}
\end{align}
for $1\leq l\leq L$ where $h_0=x$ and $y=h_L$; $d_0=d_L=1$.
We also require $\sum_j w_{ij}=1$, $\sum_ju_{kj}=1$ for all $i,k$, and all parameters except $b$ to be positive. 

We use either DSF (Equation \ref{eq:sf}) or DDSF (Equation \ref{eq:ddsf}) to define the transformer function $\tau$ in Equation \ref{eq:f}. 
To compute the log-determinant of Jacobian in a numerically stable way,
we need to apply log-sum-exp to the chain rule
\begin{align}
\nabla_xy=
\left[\nabla_{h^{(L-1)}}h^{(L)}\right]
\left[\nabla_{h^{(L-2)}}h^{(L-1)}\right]
, \cdots, 
\left[\nabla_{h^{(0)}}h^{(1)}\right]
\end{align}
We elaborate more on the numerical stability in parameterization and computation of logarithmic operations in the supplementary materials.

\subsection{Efficient Parametrization of Larger Transformers}
\label{subsec:efficient}
Multi-layer NAFs, such as DDSF, require $c$ to output $\mathcal{O}(d^2)$ pseudo-parameters, where $d$ is the number of hidden units in each layer of $\tau$.
As this is impractical for large $d$, we propose parametrizing $\tau$ with $\mathcal{O}(d^2)$ statistical parameters, but only $\mathcal{O}(d)$ pseudo-parameters which modulate the computation on a per-unit basis, using a technique such as conditional batch-normalization (CBN) \citep{Dumoulin2016}.
Such an approach also makes it possible to use minibatch-style matrix-matrix products for the forward and backwards propagation through the graph of $\tau_c$.
In particular, we use a technique similar to {\it conditional weight normalization (CWN)} \citep{krueger2017bayesian} in our experiments with DDSF; see appendix for details.

%%%%%%%%%%%%%%%%%%%%%%%%%%%%%%%%%%%%%%%%%
\subsection{Possibilities for Alternative Architectures}
\label{subsec:alternatives}
While the DSF and DDSF architectures performed well in our experiments, there are many alternatives to be explored.
One possibility is using other (strictly) monotonic activation functions in $\tau_c$, such as leaky ReLUs (LReLU) \citep{xu2015empirical} or ELUs \citep{clevert2015fast}.
Leaky ReLUs in particular are bijections on $\R$ and so would not require the softmax-weighted summation and activation function
inversion tricks discussed in the previous section.

Finally, we emphasize that in general, $\tau$ need not be expressed as a neural architecture; it only needs to satisfy the requirements of invertibility and differentiability given at the end of section 2.

%%%%%%%%%%%%%%%%%%%%%%%%%%%%%%%%%%%%%%%%%%%%%%%%%%%%%%%%%%%%%%%%%
%%%%%%%%%%%%%%%%%%%%%%%%%%%%%%%%%%%%%%%%%%%%%%%%%%%%%%%%%%%%%%%%%
\section{NAFs are Universal Density Approximators}
\label{sec:universal}

In this section, we prove that NAFs (specifically DSF) can be used to approximate any probability distribution over real vectors arbitrarily well, given that $\tau_c$ has enough hidden units output by generic neural networks with autoregressive conditioning. 
Ours is the first such result we are aware of for finite normalizing flows.

Our result builds on the work of \citet{huang2017learnable}, who demonstrate the general universal representational capability of inverse autoregressive transformations parameterized by an autoregressive neural network (that transform uniform random variables into any random variables in reals).
However, we note that their proposition is weaker than we require, as there are no constraints on the parameterization of the transformer $\tau$, whereas we've constrained $\tau$ to have strictly positive weights and monotonic activation functions, to ensure it is invertible throughout training.

The idea of proving the universal approximation theorem for DSF (1) in the IAF direction (which transforms unstructured random variables into structured random variables) resembles the concept of the {\bf inverse transform sampling}: we first draw a sample from a simple distribution, such as uniform distribution, and then pass the sample though DSF. 
If DSF converges to any inverse conditional CDF, the resulting random variable then converges in distribution to any target random variable as long as the latter has positive continuous probability density everywhere in the reals. 
(2) For the MAF direction, DSF serves as a solution to the non-linear independent component analysis problem~\cite{hyvarinen1999nonlinear}, which disentangles structured random variables into uniformly and independently distributed random variables. 
(3) Combining the two, we further show that DSF can transform any structured noise variable into a random variable with any desired distribution.

We define the following notation for the pre-logit of the DSF transformation (compare equation \ref{eq:sf}): % in the results:
\begin{equation}
\SSS(x_t,\CCC(x_{1:t-1})) = \sum_{j=1}^{n}w_j(x_{1:t-1})\cdot\sigma\left(\frac{x_t-b_j(x_{1:t-1})}{\tau_j(x_{1:t-1})}\right)
\label{eq:prelogit-dsf}
\end{equation}
where $\mathcal{C}=(w_j,b_j,\tau_j)_{j=1}^n$ are functions of $x_{1:1-t}$ parameterized by neural networks. 
Let $b_j$ be in $(r_0,r_1)$; $\tau_j$ be bounded and positive; $\sum_{j=1}^nw_j=1$ and $w_j>0$. 
See Appendix~\ref{sec:proof_of_lemmas}~and~\ref{sec:main_proof}  for the proof.

\vspace{5pt}

\begin{proposition}
\label{theorem:unstruct2struct}
\textnormal{(DSF universally transforms uniform random variables into any desired random variables)} Let $Y$ be a random vector in $\R^m$ and assume $Y$ has a strictly positive and continuous probability density distribution.
Let $X\sim\textnormal{Unif}((0,1)^m)$. 
Then there exists a sequence of functions $(G_n)_{n\geq1}$ parameterized by autoregressive neural networks in the following form 
\begin{align}
G(\x)_t = \sigma^{-1}\left(\SSS\left(x_t;\CCC_{t}(x_{1:t-1})\right)\right) 
\end{align}
where $\CCC_{t}=(a_{tj},b_{tj},\tau_{tj})_{j=1}^n$ are functions of $x_{1:t-1}$, such that $Y_n \doteq G_n(X)$ converges in distribution to $Y$.
\end{proposition}
\begin{proposition}
\label{theorem:struct2unstruct}
\textnormal{(DSF universally transforms any random variables into uniformly distributed random variables)} Let $X$ be a random vector in an open set $\mathcal{U}\subset\R^m$. 
Assume $X$ has a positive and continuous probability density distribution. 
Let $Y\sim\textnormal{Unif}((0,1)^m)$.
Then there exists a sequence of functions $(H_n)_{n\geq1}$ parameterized by autoregressive neural networks in the following form
\begin{align}
H(\x)_t = \SSS\left(x_t;\CCC_{t}(x_{1:t-1})\right) 
\end{align}
where $\CCC_{t}=(a_{tj},b_{tj},\tau_{tj})_{j=1}^n$ are functions of $x_{1:t-1}$, such that $Y_n \doteq H_n(X)$ converges in distribution to $Y$.\end{proposition}
\begin{theorem}
\label{theorem:struct2struct}
\textnormal{(DSF universally transforms any random variables into any desired random variables)} Let $X$ be a random vector in an open set $\mathcal{U}\subset\R^m$. 
Let $Y$ be a random vector in $\R^m$.
Assume both $X$ and $Y$ have a positive and continuous probability density distribution. 
Then there exists a sequence of functions $(K_n)_{n\geq1}$  parameterized by autoregressive neural networks in the following form
\begin{align}
K(\x)_t = \sigma^{-1}\left(\SSS\left(x_t;\CCC_{t}(x_{1:t-1})\right)\right)
\end{align}
where $\CCC_{t}=(a_{tj},b_{tj},\tau_{tj})_{j=1}^n$ are functions of $x_{1:t-1}$, such that $Y_n \doteq K_n(X)$ converges in distribution to $Y$.
\end{theorem}

% TODO: 
% These proofs provide some motivation for our choice of DSF (etc...) 

%%%%%%%%%%%%%%%%%%%%%%%%%%%%%%%%%%%%%%%%%%%%%%%%%%%%%%%%%%%%%%%%%
%%%%%%%%%%%%%%%%%%%%%%%%%%%%%%%%%%%%%%%%%%%%%%%%%%%%%%%%%%%%%%%%%
\section{Related work}

% 0. comparisons (AAF, hypernets)
Neural autoregressive flows are a generalization of the affine autoregressive flows introduced by \citet{kingma2016improved} as {\bf inverse autoregressive flows (IAF)} and further developed by \citet{chen2016variational} and \citet{papamakarios2017masked} as {\bf autoregressive flows (AF)} and {\bf masked autoregressive flows (MAF)}, respectively; for details on their relationship to our work see Sections 2 and 3.
% 0b. AR-NFs
While \citet{Dinh2014} draw a particular connection between their {\bf NICE } model and the {\bf Neural Autoregressive Density Estimator (NADE)} \citep{larochelle2011}, \citep{kingma2016improved} were the first to highlight the general approach of using autoregressive models to construct normalizing flows.
\citet{chen2016variational} and then \citet{papamakarios2017masked} subsequently noticed that this same approach could be used efficiently in reverse when the key operation is evaluating, as opposed to sampling from, the flow's learned output density.
Our method increases the expressivity of these previous approaches by using a neural net to output pseudo-parameters of another network, thus falling into the hypernetwork framework \citep{ha2016hypernetworks,Bertinetto2016,Brabandere2016}.

% 1. history of interest in NFs; VI/NICE
There has been a growing interest in normalizing flows (NFs) in the deep learning community, driven by successful applications and structural advantages they have over alternatives.
\citet{rippel2013high}, \citet{rezende2015variational} and \citet{Dinh2014} first introduced normalizing flows to the deep learning community as density models, variational posteriors and generative models, respectively.
In contrast to traditional variational posteriors, NFs can represent a richer family of distributions without requiring approximations (beyond Monte Carlo estimation of the KL-divergence). % \citep{rezende2015variational,kingma2016improved}.
The NF-based {\bf RealNVP}-style generative models \citep{dinh2016density,Dinh2014} also have qualitative advantages over alternative approaches.
Unlike {\bf generative adversarial networks (GANs)} \citep{goodfellow2014generative} and {\bf varational autoencoders (VAEs)} \citep{kingma2013auto,rezende2014stochastic}, computing likelihood is cheap.
Unlike autoregressive generative models, such as {\bf pixelCNNs} \citep{oord2016pixel}, sampling is also cheap.
Unfortunately, in practice RealNVP-style models are not currently competitive with autoregressive models in terms of likelihood, perhaps due to the more restricted nature of the transformations they employ.

% 2. Parallel WaveNet, MAF
Several promising recent works expand the capabilities of NFs for generative modeling and density estimation, however.
Perhaps the most exciting example is \citet{oord2017parallel}, who propose the {\bf probability density distillation} technique to train an IAF \citep{kingma2016improved} based on the autoregressive {\bf WaveNet} \citep{WaveNet} as a generative model using another pretrained WaveNet model to express the target density, thus overcoming the slow sequential sampling procedure required by the original WaveNet (and characteristic of autoregressive models in general), and reaching super-real-time speeds suitable for production.
%While \citet{Dinh2014,dinh2016density} previously developed NFs suitable for rapid generation, their techniques are not currently competive with autoregressive models.
The previously mentioned MAF technique \citep{papamakarios2017masked} further demonstrates the potential of NFs to improve on state-of-the-art autoregressive density estimation models; such highly performant MAF models could also be ``distilled'' for rapid sampling using the same procedure as in \citet{oord2017parallel}.

% 3. New applications
Other recent works also find novel applications of NFs, demonstrating their broad utility.
\citet{loaiza2017maximum} use NFs to solve maximum entropy problems, rather than match a target distribution.
\citet{louizos2017multiplicative} and \citet{krueger2017bayesian} apply NFs to express approximate posteriors over parameters of neural networks.
\citet{song2017nice} use NFs as a proposal distribution in a novel Metropolis-Hastings MCMC algorithm.

% 4. other new NF techniques
Finally, there are also several works which develop new techniques for constructing NFs that are orthogonal to ours \citep{tomczak2017improving,tomczak2016improving,gemici2016normalizing,duvenaud2016early,berg2018sylvester}.

\begin{table}
\caption{Using DSF to improve variational inference. We report the number of affine IAF with our implementation. We note that the log likelihood reported by \citet{kingma2016improved} is $78.88$. The average and standard deviation are carried out with 5 trials of experiments with different random seeds.}
\label{tb:vae}
\centering
\begin{tabular}{lcc}
\toprule
Model& ELBO & $\log p(x)$\\
\midrule
VAE & $85.00\pm0.03$ & $81.66\pm0.05$\\
IAF-affine & $82.25\pm0.05$ & $80.05\pm0.04$ \\
IAF-DSF  & $81.92\pm0.04$ & $79.86\pm0.01$\\
\bottomrule
\end{tabular}
\end{table}

\begin{table*}[tbh]
\caption{
Test log-likelihood and error bars of 2 standard deviations on the 5 datasets (5 trials of experiments).
Neural autoregressive flows (NAFs) produce state-of-the-art density estimation results on all 5 datasets.
The numbers (5 or 10) in parantheses indicate the number of transformations which were stacked; for TAN \citep{TAN}, we include their best results, achieved using different architectures on different datasets.
We also include validation results to give future researchers a fair way of comparing their methods with ours during development.
}
\vspace{1mm}
\label{tb:maf}
\centering
\begin{tabular}{lccccc}
\toprule
Model& POWER & GAS & HEPMASS & MINIBOONE & BSDS300\\
\midrule
MADE MoG & $0.40\pm0.01$ & $8.47\pm0.02$ & $-15.15\pm0.02$ & $-12.27\pm0.47$ & $153.71\pm0.28$ \\
MAF-affine (5) & $0.14\pm0.01$ & $9.07\pm0.02$ & $-17.70\pm0.02$ & $-11.75\pm0.44$ & $155.69\pm0.28$ \\
MAF-affine (10)  & $0.24\pm0.01$ & $10.08\pm0.02$ & $-17.73\pm0.02$ & $-12.24\pm0.45$ & $154.93\pm0.28$\\
MAF-affine MoG (5) & $0.30\pm0.01$ & $9.59\pm0.02$ & $-17.39\pm0.02$ & $-11.68\pm0.44$ & $156.36\pm0.28$ \\
\midrule
TAN (various architectures) & $0.48\pm0.01$ & $11.19\pm0.02$ & $-15.12\pm0.02$ & $-11.01\pm0.48$ & $157.03\pm0.07$ \\
\midrule
MAF-DDSF (5)  & $\mathbf{0.62 \pm 0.01} $ & $11.91 \pm 0.13$  & $\mathbf{-15.09 \pm 0.40}$ & $\mathbf{-8.86 \pm 0.15}$ & $\mathbf{157.73 \pm 0.04}$ \\
MAF-DDSF (10) & $0.60 \pm 0.02$           & $\mathbf{11.96 \pm 0.33}$           & $        -15.32 \pm 0.23 $ & $        -9.01 \pm 0.01 $ & $        157.43 \pm 0.30$ \\
\midrule
MAF-DDSF (5)  valid & $0.63 \pm 0.01$ & $11.91 \pm 0.13$ & $ 15.10 \pm 0.42$ & $-8.38 \pm 0.13$ & $172.89 \pm 0.04$ \\
MAF-DDSF (10) valid & $0.60 \pm 0.02$ & $11.95 \pm 0.33$ & $ 15.34 \pm 0.24$ & $-8.50 \pm 0.03$ & $172.58 \pm 0.32$ \\
\bottomrule
\end{tabular}
\end{table*}

\begin{figure}
\label{fig:grid}
\centering
\includegraphics[width=0.46\textwidth, height=0.15\textwidth]{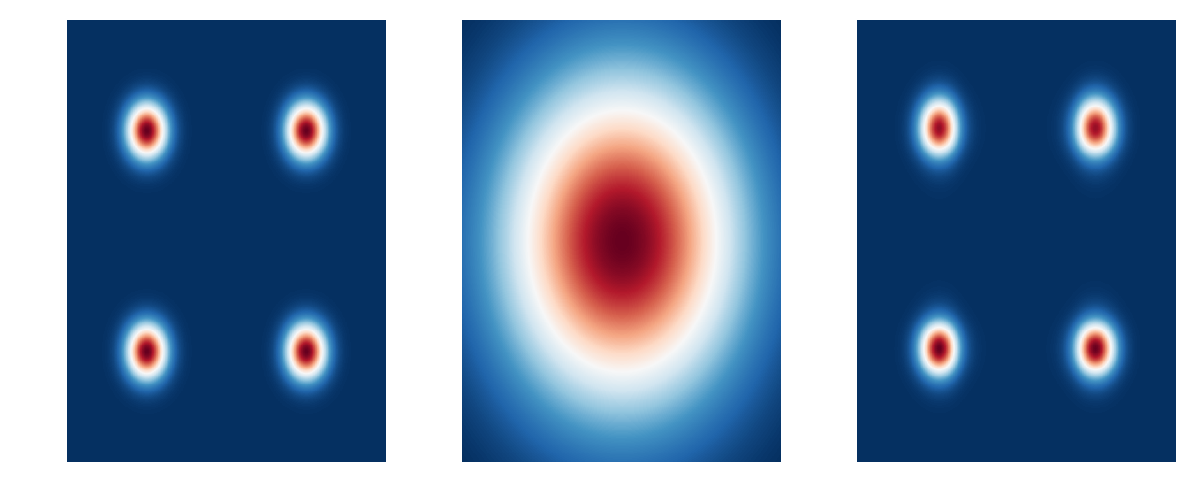}
\includegraphics[width=0.46\textwidth, height=0.15\textwidth]{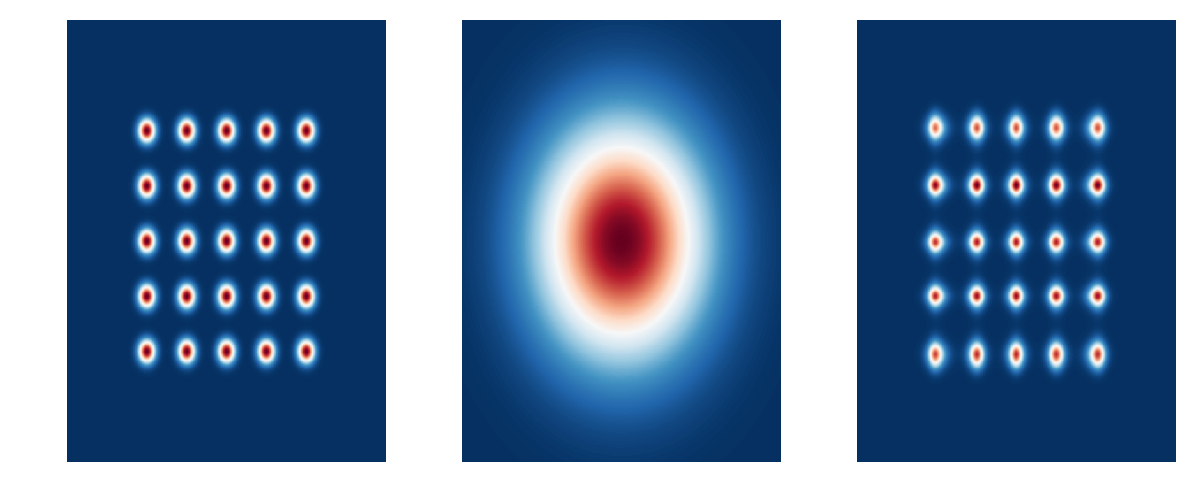}
\includegraphics[width=0.46\textwidth, height=0.15\textwidth]{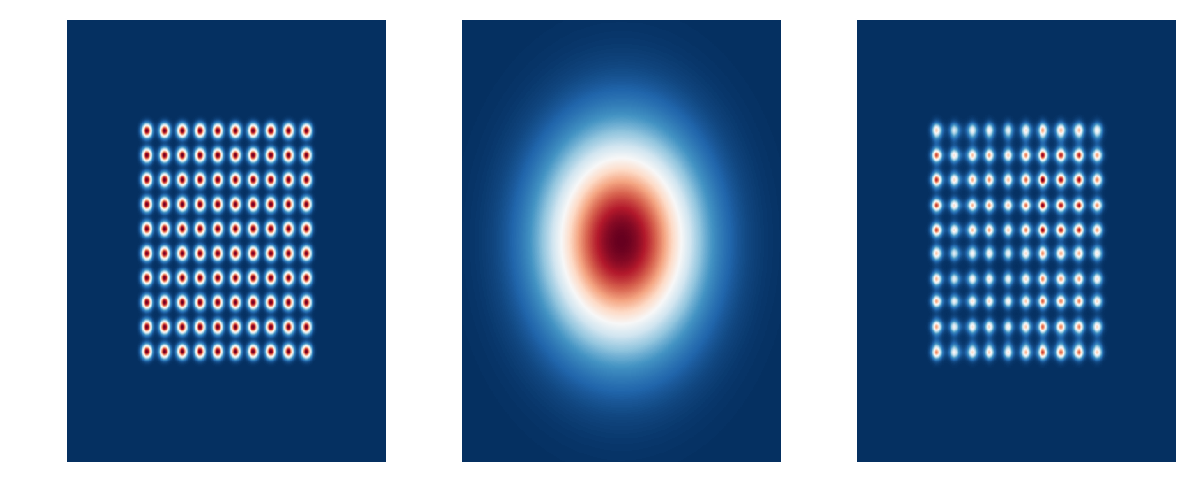}
\caption{Fitting grid of Gaussian distributions using maximum likelihood. Left: true distribution. Center: affine autoregressive flow (AAF). Right: neural autoregressive flow (NAF)}
\end{figure}

\begin{figure}
\centering
\includegraphics[width=0.23\textwidth]{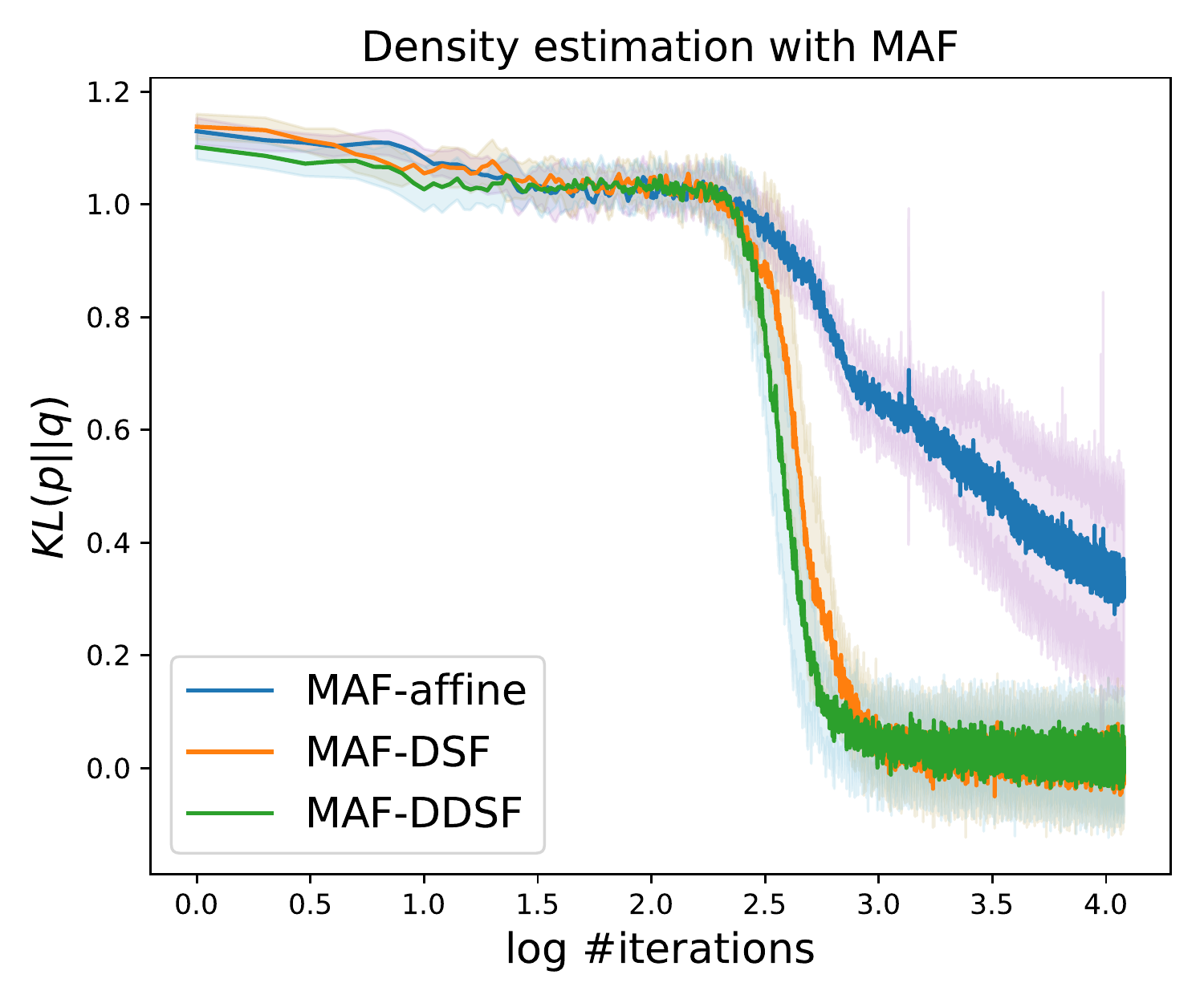}
\includegraphics[width=0.23\textwidth]{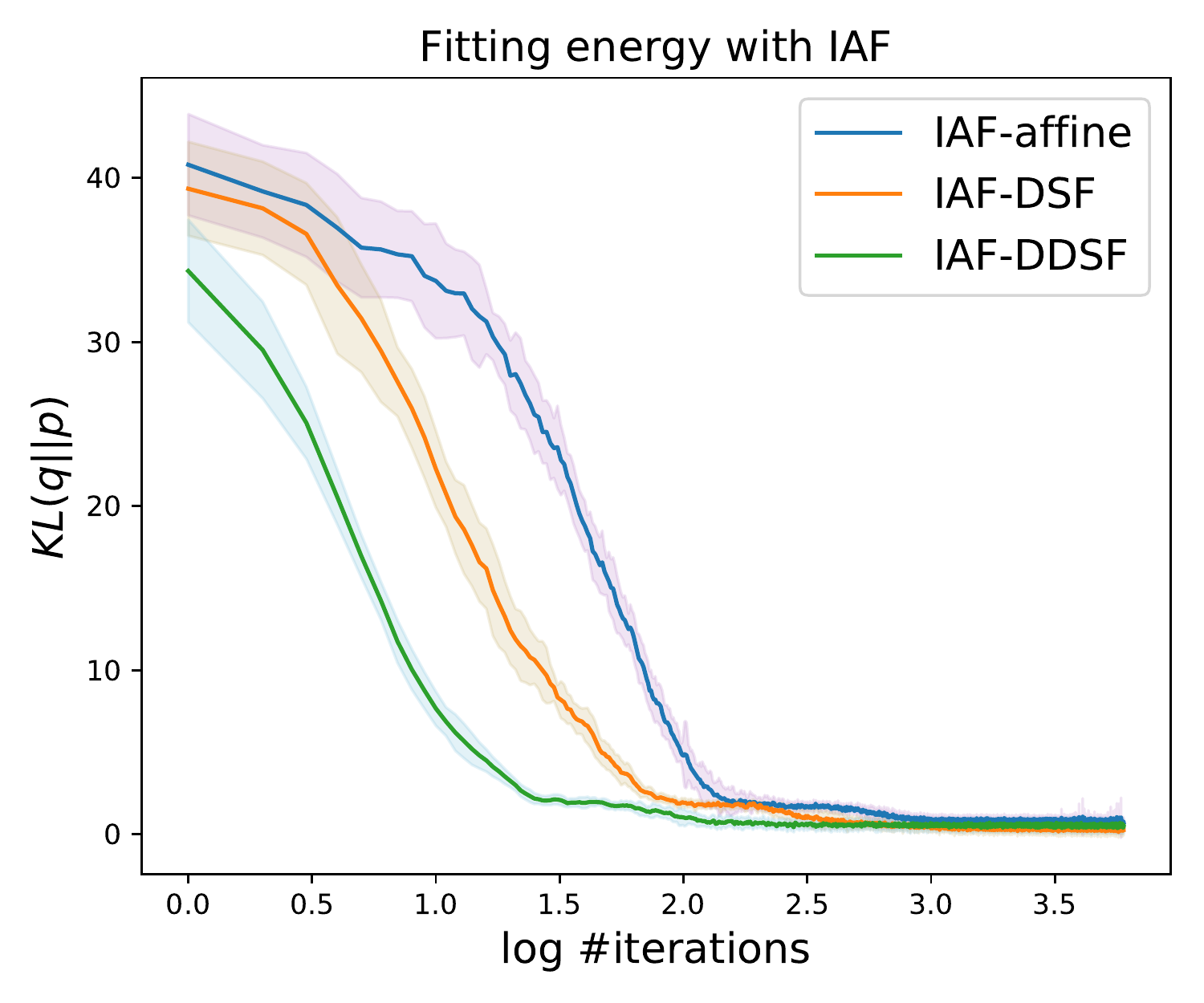}
\caption{Learning curve of MAF-style and IAF-style training. $q$ denotes our trained model, and $p$ denotes the target.}
\end{figure}

\begin{figure}
\centering
 \includegraphics[width=0.23\textwidth]{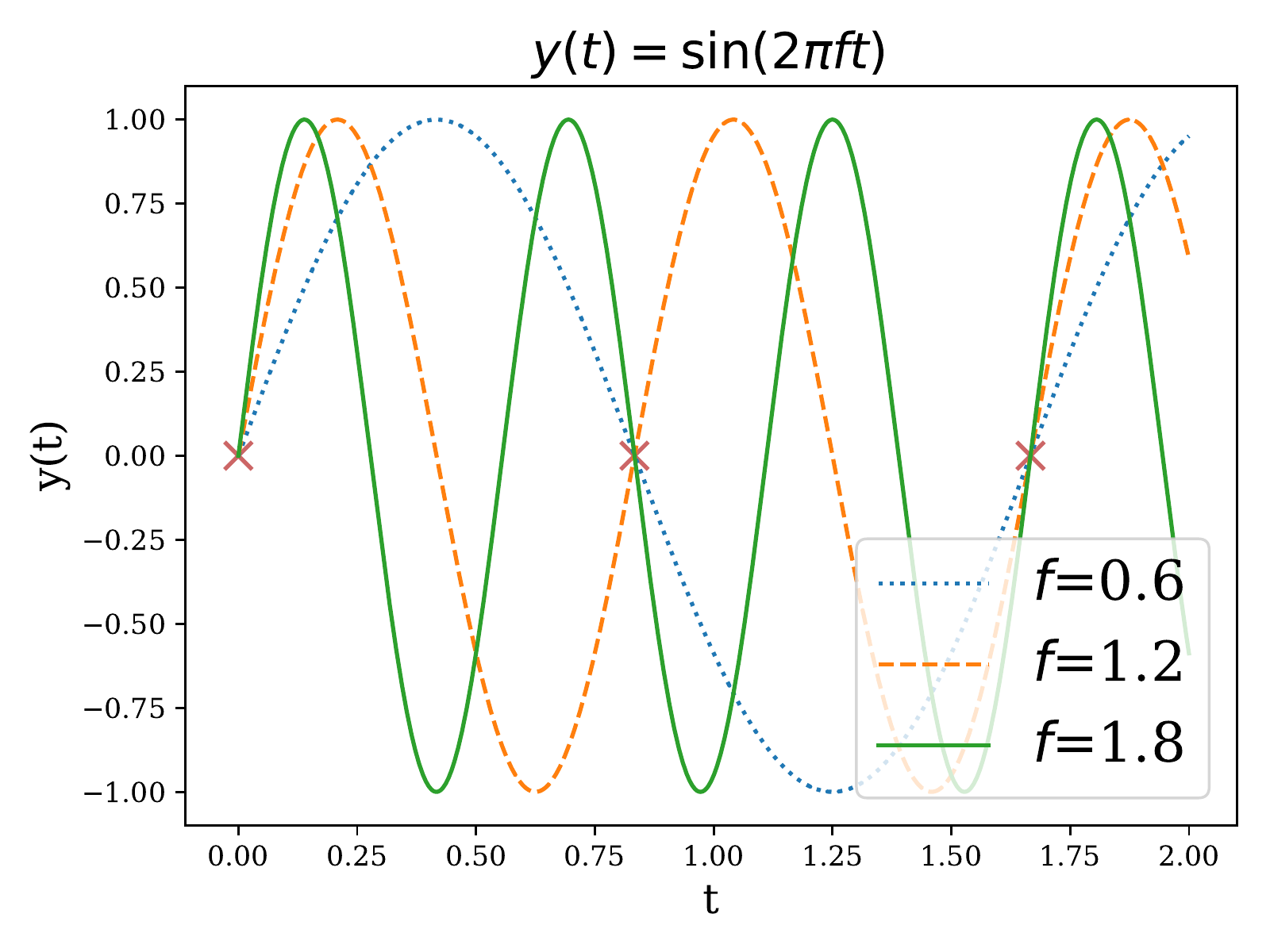}
\hfill
\includegraphics[width=0.23\textwidth]{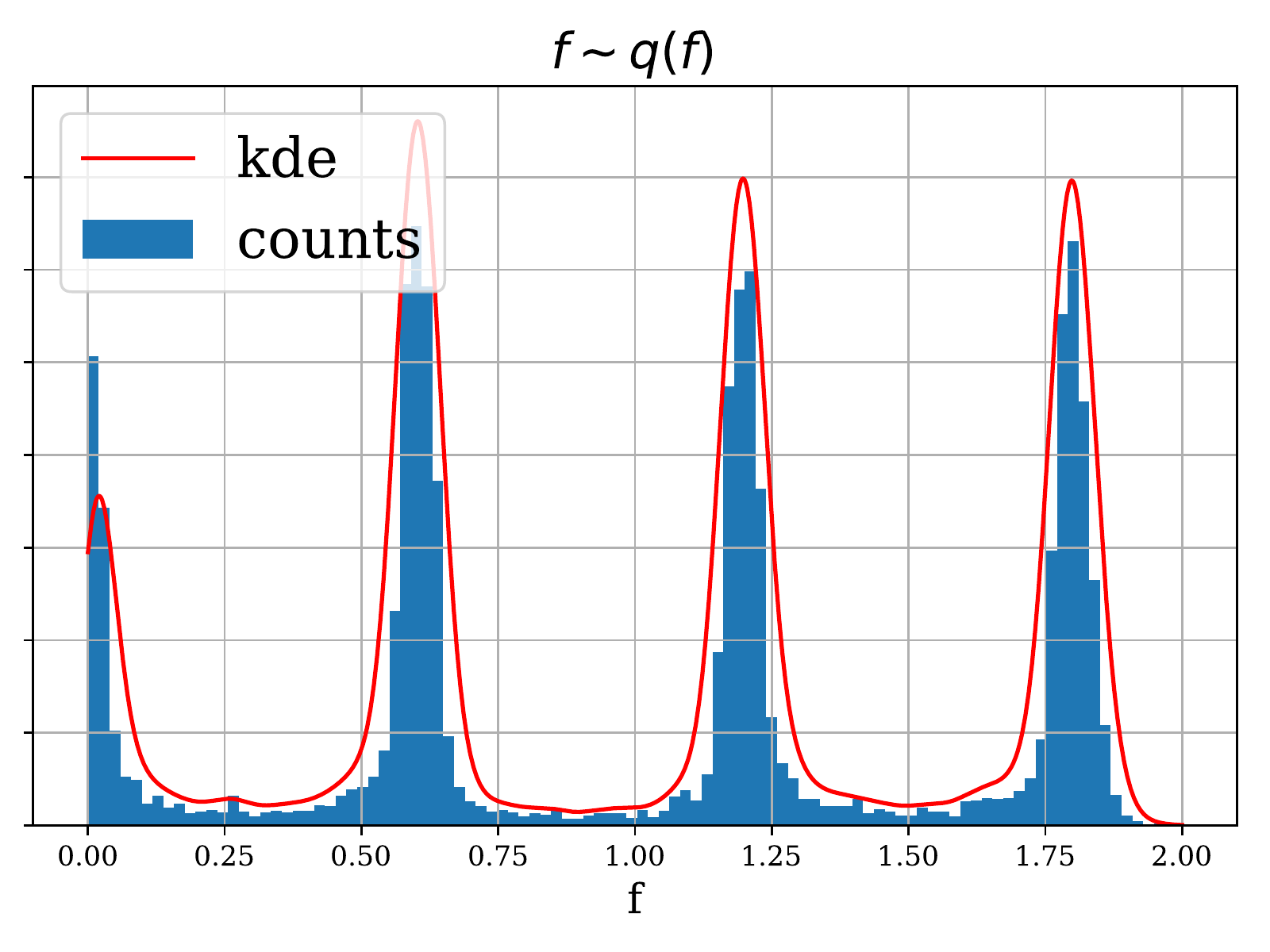}
\caption{
The DSF model effectively captures the true posterior distribution over the frequency of a sine wave.
Left: The three observations (marked with red x's) are compatible with sine waves of frequency $f \in {0.0, 0.6, 1.2, 1.8}$.  Right:  a histogram of samples from the DSF approximate posterior (``counts'') and a Kernel Density Estimate of the distribution it represents (KDE).  
}
\label{fig:sw}
\end{figure}

%\todo{use consistent notation! (NAF instead of DSF)}

%%%%%%%%%%%%%%%%%%%%%%%%%%%%%%%%%%%%%%%%%%%%%%%%%%%%%%%%%%%%%%%%%
%%%%%%%%%%%%%%%%%%%%%%%%%%%%%%%%%%%%%%%%%%%%%%%%%%%%%%%%%%%%%%%%%
\section{Experiments}
Our experiments evaluate NAFs on the classic applications of variational inference and density estimation, where we outperform IAF and MAF baselines.
% TODO: here
We first demonstrate the qualitative advantage NAFs have over AAFs in energy function fitting and density estimation (Section~\ref{subsec:toy}). 
We then demonstrate the capability of NAFs to capture a multimodal Bayesian posterior in a limited data setting (Section~\ref{subsec:sine}). 
For larger-scale experiments, we show that using NAF instead of IAF to approximate the posterior distribution of latent variables in a variational autoencoder \citep{kingma2013auto,rezende2014stochastic} yields better likelihood results on binarized MNIST \cite{larochelle2011} (Section~\ref{subsec:VAE}). % TODO: \citep{}.
Finally, we report our experimental results on density estimation of a suite of UCI datasets (Section~\ref{subsec:MAF}).

%%%%%%%%%%%%%%%%%%%%%%%%%%%%%%% 
\subsection{Toy energy fitting and density estimation}
\label{subsec:toy}
\subsubsection{Expressiveness}
%(\textit{Expressiveness}) 
First, we demonstrate that, in the case of marginally independent distributions, affine transformation can fail to fit the true distribution. 
We consider a mixture of Gaussian density estimation task. 
We define the modes of the Gaussians to be laid out on a 2D meshgrid within the range [-5,5], and consider 2, 5 and 10 modes on each dimension.
While the affine flow only produces a single mode, the neural flow matches the target distribution quite well even up to a 10x10 grid with 100 modes (see Figure~\ref{fig:grid}).
%\todo{reference figure}

\subsubsection{Convergence} 
%(\textit{Convergence}) 
We then repeat the experiment that produces Figure \ref{fig:4mode_iaf} and \ref{fig:4mode_maf} 16 times, smooth out the learning curve and present average convergence result of each model with its corresponding standard deviation. 
For affine flow, we stack 6 layers of transformation with reversed ordering.
For DSF and DDSF we used one transformation.
We set $d=16$ for both, $L=2$ for DDSF.

%%%%%%%%%%%%%%%%%%%%%%%%%%%%%%% 
% TODO: is this really relevant given the other toy experiments??
\subsection{Sine Wave experiment}
\label{subsec:sine}
Here we demonstrate the ability of DSF to capture multimodal posterior distributions.
To do so, we create a toy experiment where the goal is to infer the posterior over the frequency of a sine wave, given only 3 datapoints.
We fix the form of the function as $y(t)=\sin(2\pi f\cdot t)$
and specify a Uniform prior over the frequency: $p(f) = U([0,2])$.
The task is to infer the posterior distribution $p(f|T,Y)$ given the dataset $(T,Y) = ((0,\vfrac{5}{6},\vfrac{10}{6}),(0,0,0))$, as represented by the red crosses of Figure \ref{fig:sw} (left).
We assume the data likelihood given the frequency parameter to be $p(y_i|t_i,f)=\mathcal{N}(y_i;y_f(t_i), 0.125)$, where the variance $\sigma^2=0.125$ represents the inherent uncertainty of the data. 
Figure \ref{fig:sw} (right) shows that DSF learns a good posterior in this task.

%%%%%%%%%%%%%%%%%%%%%%%%%%%%%%% 
\subsection{Amortized Approximate Posterior}
\label{subsec:VAE}
%todo{DK says: I don't understand this part.  Explain it!}
We evaluate NAF's ability to improve variational inference, in the context of the binarized MNIST \citep{larochelle2011} benchmark using the well-known variational autoencoder \citep{kingma2013auto,rezende2014stochastic} (Table \ref{tb:vae}).
Here again the DSF architecture outperforms both standard IAF and the traditional independent Gaussian posterior by a statistically significant margin.

%%%%%%%%%%%%%%%%%%%%%%%%%%%%%%% UCI
\subsection{Density Estimation with Masked Autoregressive Flows}
\label{subsec:MAF}

We replicate the density estimation experiments of \citet{papamakarios2017masked}, which compare MADE~\citep{MADE} and RealNVP~\citep{dinh2016density} to their proposed MAF model (using either 5 or 10 layers of MAF) on BSDS300 \citep{BSDS} as well as 4 UCI datasets \citep{UCI} processed as in \citet{uria2013rnade}.
Simply replacing the affine transformer with our DDSF architecture in their best performing architecture for each task (keeping all other settings fixed) results in substantial performance gains, and also outperforms the more recent Transformation Autoregressive Networks (TAN) \citet{TAN}, setting a new state-of-the-art for these tasks. 
Results are presented in Table~\ref{tb:maf}.

%%%%%%%%%%%%%%%%%%%%%%%%%%%%%%%%%%%%%%%%%%%%%%%%%%%%%%%%%%%%%%%%%
%%%%%%%%%%%%%%%%%%%%%%%%%%%%%%%%%%%%%%%%%%%%%%%%%%%%%%%%%%%%%%%%%
\section{Conclusion} % TODO at the end 
In this work we introduce the neural autoregressive flow (NAF), a flexible method of tractably approximating rich families of distributions. 
In particular, our experiments show that NAF is able to model multimodal distributions and outperform related methods such as inverse autoregressive flow in density estimation and variational inference.
Our work emphasizes the difficulty and importance of capturing multimodality, as previous methods fail even on simple toy tasks, whereas our method yields significant improvements in performance.

%%%%%%%%%%%%%%%%%%%%%%%%%%%%%%%%%%%%%%%%%%%%%%%%%%%%%%%%%%%%%%%%%
%%%%%%%%%%%%%%%%%%%%%%%%%%%%%%%%%%%%%%%%%%%%%%%%%%%%%%%%%%%%%%%%%
\section*{Acknowledgements}
We would like to thank Tegan Maharaj, Ahmed Touati, Shawn Tan and Giancarlo Kerg for helpful comments and advice. 
We also thank George Papamakarios for providing details on density estimation task's setup. 

%\newpage
%\clearpage 
\FloatBarrier

% references
\bibliography{example_paper}
\bibliographystyle{icml2018}
\clearpage
\appendix

% This document was modified from the file originally made available by
% Pat Langley and Andrea Danyluk for ICML-2K. This version was created
% by Iain Murray in 2018. It was modified from a version from Dan Roy in
% 2017, which was based on a version from Lise Getoor and Tobias
% Scheffer, which was slightly modified from the 2010 version by
% Thorsten Joachims & Johannes Fuernkranz, slightly modified from the
% 2009 version by Kiri Wagstaff and Sam Roweis's 2008 version, which is
% slightly modified from Prasad Tadepalli's 2007 version which is a
% lightly changed version of the previous year's version by Andrew
% Moore, which was in turn edited from those of Kristian Kersting and
% Codrina Lauth. Alex Smola contributed to the algorithmic style files.

%%%%%%%%%%%%%%%%%%%%%%%%%%%%%%%%%%%%%%%%%%
%%%%%%%%%%%%%%%%%%%%%%%%%%%%%%%%%%%%%%%%%%
%-%  Monotonicity of NAF
%%%%%%%%%%%%%%%%%%%%%%%%%%%%%%%%%%%%%%%%%%
%%%%%%%%%%%%%%%%%%%%%%%%%%%%%%%%%%%%%%%%%%
\section{Exclusive KL View of the MLE}
\label{sec:KLnMLE}
Lets assume a change-of-variable model $p_Z(\z)$ on the random variable $Z\in\R^m$, such as the one used in \citet{dinh2016density}: $\mathbf{z_0}\sim p_0(\mathbf{z_0})$ and $\z=\psi(\mathbf{z_0})$, where $\psi$ is an invertible function and density evaluation of $\mathbf{z_0}\in\R^m$ is tractable under $p_0$. 
The resulting density function can be written
$$p_Z(\z)=p_0(\mathbf{z_0})\left| \frac{\partial \psi(\mathbf{z_0})}{\partial \mathbf{z_0}} \right|^{-1}$$
The maximum likelihood principle requires us to minimize the following metric:
\begin{align*}
\KL&\big( p_{data}(\z) || p_Z(\z) \big) \\
&= \E_{p_{data}}\left[\log p_{data}(\z) - \log p_Z(\z) \right] \\
&= \E_{p_{data}}\left[\log p_{data}(\z) - \log p_0(\mathbf{z_0})\left| \frac{\partial \psi^{-1}(\z)}{\partial \z} \right| \right] \\
&= \E_{p_{data}}\left[\log p_{data}(\z)\left| \frac{\partial \psi^{-1}(\z)}{\partial \z} \right|^{-1} - \log p_0(\mathbf{z_0}) \right] 
\intertext{which coincides with exclusive KL divergence; 
to see this, we take $X=Z$, $Y=Z_0$, $f=\psi^{-1}$, $p_X=p_{data}$, and $p_{target}=p_0$}
&= \E_{p_{X}}\left[\log p_{X}(\x)\left| \frac{\partial f(\x)}{\partial \x} \right|^{-1} - \log p_{target}(\y) \right] \\
&= \KL\big( p_Y(\y) || p_{target}(\y) \big)
\end{align*}
This means we want to transform the empirical distribution $p_{data}$, or $p_X$, to fit the target density (the usually unstructured, base distribution $p_0$), as explained in Section~\ref{sec:background}.

%%%%%%%%%%%%%%%%%%%%%%%%%%%%%%%%%%%%%%%%%%
%%%%%%%%%%%%%%%%%%%%%%%%%%%%%%%%%%%%%%%%%%
%-%  Monotonicity of NAF
%%%%%%%%%%%%%%%%%%%%%%%%%%%%%%%%%%%%%%%%%%
%%%%%%%%%%%%%%%%%%%%%%%%%%%%%%%%%%%%%%%%%%
\section{Monotonicity of NAF}
Here we show that using \textit{strictly positive weights} and \textit{strictly monotonically increasing activation functions} is sufficient to ensure strict monotonicity of NAF.
For brevity, we write \textit{monotonic}, or \textit{monotonicity} to represent the strictly monotonically increasing behavior of a function.
Also, note that a continuous function is strictly monotonically increasing exactly when its derivative is  greater than 0.

\begin{proof} \textnormal{(Proposition~\ref{proposition:mono})}

Suppose we have an MLP with $L+1$ layers: $h_0, h_1, h_2, ... , h_L$, $x=h_0$ and $y=h_L$, where $x$ and $y$ are scalar, and $h_{l,j}$ denotes the $j$-th node of the $l$-th layer.% after the input layer.
For $1\leq l\leq L$, we have
\begin{align}
p_{l,j} &= w_{l,j}^Th_{l-1}+b_{l,j} \\
h_{l,j} &= A_{l}(p_{l,j})
\end{align}
for some monotonic activation function $A_l$, positive weight vector $w_{l,j}$, and bias $b_{l,j}$. 
Differentiating this yields
\begin{align}
\frac{\mathrm{d}h_{l,j}}{\mathrm{d}h_{l-1,k}} = \frac{\mathrm{d}A_{l}(p_{l,j})}{\mathrm{d}p_{l,j}}\cdot w_{l,j,k},% > 0 
\label{eq:unit2unit}
\end{align}
which is greater than 0 since $A$ is monotonic (and hence has positive derivative), and $w_{l,j,k}$ is positive by construction.
%The inequality is due to the monotonicity of the activation and the positivity of the weight. 
Thus any unit of layer $l$ is monotonic with respect to any unit of layer $l-1$, for $1\leq l\leq L$. 

Now, suppose we have $J$ monotonic functions $f_j$ of the input $x$.
Then the weighted sum of these functions $\sum_{j=1}^J u_jf_j $ with $u_j>0$ is also monotonic with respect to $x$, since 
\begin{align}
\frac{\mathrm{d}}{\mathrm{d}x}\sum_{j=1}^J u_jf_j = \sum_{j=1}^J u_j \frac{\mathrm{d}f_j }{\mathrm{d}x} > 0
\label{eq:possum}
\end{align}

Finally, we use induction to show that all $h_{l,j}$ (including $y$) are monotonic with respect to $x$.% for any $l$ and $j$, and thus so is $y$. 
\begin{enumerate}
\item The base case ($l=1$) is given by Equation \ref{eq:unit2unit}. 
\item Suppose the inductive hypothesis holds, which means $h_{l,j}$ is monotonic with respect to $x$ for all $j$ of layer $l$. 
Then by Equation \ref{eq:possum}, % and the monotonicity of the activation function, 
$h_{l+1,k}$ is also monotonic with respect to $x$ for all k.
\end{enumerate}

Thus by mathematical induction, monotonicity of $h_{l,j}$ holds for all $l$ and $j$. 
\end{proof}

%%%%%%%%%%%%%%%%%%%%%%%%%%%%%%%%%%%%%%%%%%
%%%%%%%%%%%%%%%%%%%%%%%%%%%%%%%%%%%%%%%%%%
%-%  Log determinant of Jacobian
%%%%%%%%%%%%%%%%%%%%%%%%%%%%%%%%%%%%%%%%%%
%%%%%%%%%%%%%%%%%%%%%%%%%%%%%%%%%%%%%%%%%%
\section{Log Determinant of Jacobian}
As we mention at the end of Section~\ref{subsec:arch}, to compute the log-determinant of the Jacobian as part of the objective function, we need to handle the numerical stability.
We first derive the Jacobian of the DDSF (note that DSF is a special case of DDSF), and then summarize the numerically stable operations that were utilized in this work.

\subsection{Jacobian of DDSF}
Again defining $x=h_0$ and $y=h_L$, the Jacobian of each DDSF transformation can be written as a sequence of dot products due to the chain rule:
\begin{align}
\nabla_xy=
\left[\nabla_{h^{(L-1)}}h^{(L)}\right]
\left[\nabla_{h^{(L-2)}}h^{(L-1)}\right]
, \cdots, 
\left[\nabla_{h^{(0)}}h^{(1)}\right]
\label{eq:app-chain}
\end{align} 

For notational convenience, we define a few more intermediate variables. For each layer of DDSF, we have
\begin{align}
\smallunderbrace{C^{(l+1)}}_{d_{l+1}} &=  \smallunderbrace{a^{(l+1)}}_{d_{l+1}} \odot (\smallunderbrace{u^{(l+1)}}_{d_{l+1}\times d_{l}} \cdot \smallunderbrace{h^{(l)}}_{d_{l}}) + {\smallunderbrace{b^{(1+1)}}_{d_{l+1}}} \nonumber\\
\smallunderbrace{D^{(l+1)}}_{d_{l+1}} &= \smallunderbrace{w^{(l+1)}}_{d_{l+1}\times d_{l+1}} \cdot\, \sigma(\smallunderbrace{C^{(l+1)}}_{d_{l+1}}) \nonumber\\
\smallunderbrace{h^{(l+1)}}_{d_{l+1}} ) \nonumber &= \sigma^{-1}(\smallunderbrace{D^{(l+1)}}_{d_{l+1}} ) \nonumber
\end{align}
The gradient can be expanded as

\begin{align}
\nabla_{h^{(l)}}\left(h^{(l+1)}\right) &=
\,\bigg(
\;\nabla_D\left(\sigma^{-1}(D^{(l+1)})\right)_{[:,\bullet]} \odot \nonumber \\ 
&\quad\;\;\quad\nabla_{\sigma(C)}\left(D^{(l+1)}\right) \odot \nonumber \\
&\quad\;\;\quad\nabla_C\left(\sigma(C^{(l+1)})\right)_{[\bullet,:]}  \odot \nonumber \\
&\quad\;\;\quad\nabla_{(u^{(l+1)}h^{(l)})}\left(C^{(l+1)}\right)_{[\bullet,:]} 
\bigg)_{[:,:,\bullet]}  \times_{-1} \nonumber\\
&\quad\;\;\;\nabla_{h^{(l)}}\left(u^{(l+1)}h^{(l)}\right)_{[\bullet,:,:]} \nonumber 
\intertext{where the bullet $\bullet$ in the subscript indicates the dimension is broadcasted, $\odot$ denotes element-wise multiplication, and $\times_{-1}$ denotes summation over the last dimension after element-wise product,}
&=
\,\bigg(
\;\left(\frac{1}{D^{(l+1)}(1-D^{(l+1)})}\right)_{[:,\bullet]} \odot \nonumber \\ 
&\quad\;\;\quad\left(w^{(l+1)}\right) \odot \nonumber \\
&\quad\;\;\quad\left(\sigma(C^{(l+1)})\odot(1-\sigma(C^{(l+1)}))\right)_{[\bullet,:]} \odot \nonumber \\
&\quad\;\;\quad\left(a^{(l+1)}\right)_{[\bullet,:]}
\bigg)_{[:,:,\bullet]}  \times_{-1} \nonumber\\
&\quad\;\;\;\left(u^{(l+1)}\right)_{[\bullet,:,:]} 
\label{eq:app-ddsf-jacob}
\end{align}

\subsection{Numerically Stable Operations}
Since the Jacobian of each DDSF transformation is chain of dot products (Equation \ref{eq:app-chain}), with some nested multiplicative operations (Equation \ref{eq:app-ddsf-jacob}), we calculate everything in the log-scale (where multiplication is replaced with addition) to avoid unstable gradient signal.

\subsubsection{Log Activation}
\label{sec:log_activation}
To ensure the summing-to-one and positivity constraints of $u$ and $w$, we let the autoregressive conditioner output pre-activation $u\_$ and $w\_$, and apply softmax to them. 
We do the same for $a$ by having the conditioner output $a\_$ and apply softplus to ensure positivity. 
In Equation \ref{eq:app-ddsf-jacob}, we have 
\begin{align}
\log w &= \operatorname{logsoftmax}(w\_) \nonumber \\
\log u &= \operatorname{logsoftmax}(u\_) \nonumber \\
\log \sigma(C) &= \operatorname{logsigmoid}(C) \nonumber \\
\log 1-\sigma(C) &= \operatorname{logsigmoid}(-C)    \nonumber
\end{align}
where
\begin{align}
\operatorname{logsoftmax}(x) &= x - \operatorname{logsumexp}(x) \nonumber \\
\operatorname{logsigmoid}(x) &= -\operatorname{softplus}(-x) \nonumber \\
\operatorname{logsumexp}(x) &= \log(\sum_i \exp(x_i-x^*)) + x^* \nonumber \\
\operatorname{softplus}(x) &= \log(1+\exp(x)) + \delta   \nonumber
\end{align}
where $x^*=\max_i\{x_i\}$ and $\delta$ is a small value such as $10^{-6}$.

\subsubsection{Logarithmic dot product}
In both Equation \ref{eq:app-chain} and \ref{eq:app-ddsf-jacob}, we encounter matrix/tensor product, which is achived by summing over one dimension after doing element-wise multiplication. 
Let $\tilde{M}_1=\log M_1$ and $\tilde{M}_2=\log M_2$ be $d_0\times d_1$ and $d_1\times d_2$, respectively.
The logarithmic matrix dot product $\star$ can be written as:
\begin{align}
\tilde{M}_1\star&\tilde{M}_2 = \nonumber\\
&\operatorname{logsumexp}_{\textnormal{dim}=1}\left(\left(\tilde{M}_1\right)_{[:,:,\bullet]} +
\left(\tilde{M}_2\right)_{[\bullet,:,:]}\right) \nonumber
\end{align}
where the subscript of $\operatorname{logsumexp}$ indicates the dimension (index starting from $0$) over which the elements are to be summed up. 
Note that $\tilde{M}_1\star\tilde{M}_2=\log \left(M_1\cdot M_2\right)$.

%%%%%%%%%%%%%%%%%%%%%%%%%%%%%%%%%%%%%%%%%%
%%%%%%%%%%%%%%%%%%%%%%%%%%%%%%%%%%%%%%%%%%
%-%  Scalability and Parameter Sharing
%%%%%%%%%%%%%%%%%%%%%%%%%%%%%%%%%%%%%%%%%%
%%%%%%%%%%%%%%%%%%%%%%%%%%%%%%%%%%%%%%%%%%
\section{Scalability and Parameter Sharing}
As discussed in Section \ref{subsec:alternatives}, 
a multi-layer NAF such as DDSF requires the autoregressive conditioner $c$ to output many pseudo-parameters, on the order of 
$\mathcal{O}(Ld^2)$,
where $L$ is the number of layers of the transformer network ($\tau$), and $d$ is the average number of hidden units per layer. 
In practice, we reduce the number of outputs (and thus the computation and memory requirements) of DDSF by instead endowing $\tau$ with some learned ({\it non-conditional}) statistical parameters.
Specifically, we decompose $w\_$ and $u\_$ (the preactivations of $\tau$'s weights, see section \ref{sec:log_activation}) into pseudo-parameters and statistical parameters.
Take $u\_$ for example:
$$
u^{(l+1)} = \operatorname{softmax}_{\textnormal{dim}=1}(v^{(l+1)}+\eta_{[\bullet,:]}^{(l+1)})
$$
where $v^{(l+1)}$ is a $d_{l+1}\times d_{l}$ matrix of statistical parameters, and $\eta$ is output by $c$. 
See figure \ref{fig:factorized_weight} for a depiction.

The linear transformation before applying sigmoid resembles conditional weight normalization (CWN) \cite{krueger2017bayesian}. 
While CWN rescales the weight vectors normalized to have unit L2 norm, here we rescale the weight vector normalized by softmax such that it sums to one and is positive.
We call this \textit{conditional normalized weight exponentiation}. 
This reduces the number of pseudo-parameters to $\mathcal{O}(Ld)$.

\begin{figure}
\label{fig:factorized_weight}
\centering
\includegraphics[width=0.45\textwidth]{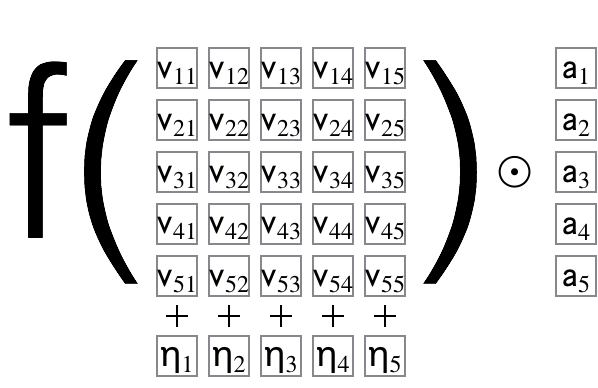}
\caption{Factorized weight of DDSF.
The $v_i,j$ are learned parameters of the model; only the pseudo-parameters $\eta$ and $a$ are output by the conditioner.
The activation function $f$ is softmax, so adding $\eta$ yields an element-wise rescaling of the inputs to this layer of the transformer by $\exp(\eta)$.
}
\end{figure}

%%%%%%%%%%%%%%%%%%%%%%%%%%%%%%%%%%%%%%%%%%
%%%%%%%%%%%%%%%%%%%%%%%%%%%%%%%%%%%%%%%%%%
%-%  Identity Flow Initialization
%%%%%%%%%%%%%%%%%%%%%%%%%%%%%%%%%%%%%%%%%%
%%%%%%%%%%%%%%%%%%%%%%%%%%%%%%%%%%%%%%%%%%
\section{Identity Flow Initialization}
In many cases, initializing the transformation to have a minimal effect is believed to help with training, as it can be thought of as a warm start with a simpler distribution. 
For instance, for variational inference, when we initialize the normalizing flow to be an identity flow, the approximate posterior is at least as good as the input distribution (usually a fully factorized Gaussian distribution) before the transformation. 
To this end, for DSF and DDSF, we initialize the pseudo-weights $a$ to be close to $1$, the pseudo-biases $b$ to be close to $0$. 

This is achieved by initializing the conditioner (whose outputs are the pseudo-parameters) to have small weights and the appropriate output biases.
Specifically, we initialize the output biases of the last layer of our MADE \citep{MADE} conditioner to be zero, and add $\operatorname{softplus}^{-1}(1)\approx0.5413$ to the outputs of which correspond to $a$ before applying the softplus activation function. 
We initialize all conditioner's weights by sampling from from $\textnormal{Unif}(-0.001,0.001)$.
We note that there might be better ways to initialize the weights to account for the different numbers of active incoming units.

%%%%%%%%%%%%%%%%%%%%%%%%%%%%%%%%%%%%%%%%%%
%%%%%%%%%%%%%%%%%%%%%%%%%%%%%%%%%%%%%%%%%%
%-%  Proof of Lemma 1
%%%%%%%%%%%%%%%%%%%%%%%%%%%%%%%%%%%%%%%%%%
%%%%%%%%%%%%%%%%%%%%%%%%%%%%%%%%%%%%%%%%%%
\section{Lemmas: Uniform Convergence of DSF}
\label{sec:proof_of_lemmas}
\begin{figure}
\centering
\includegraphics[width=0.45\textwidth]{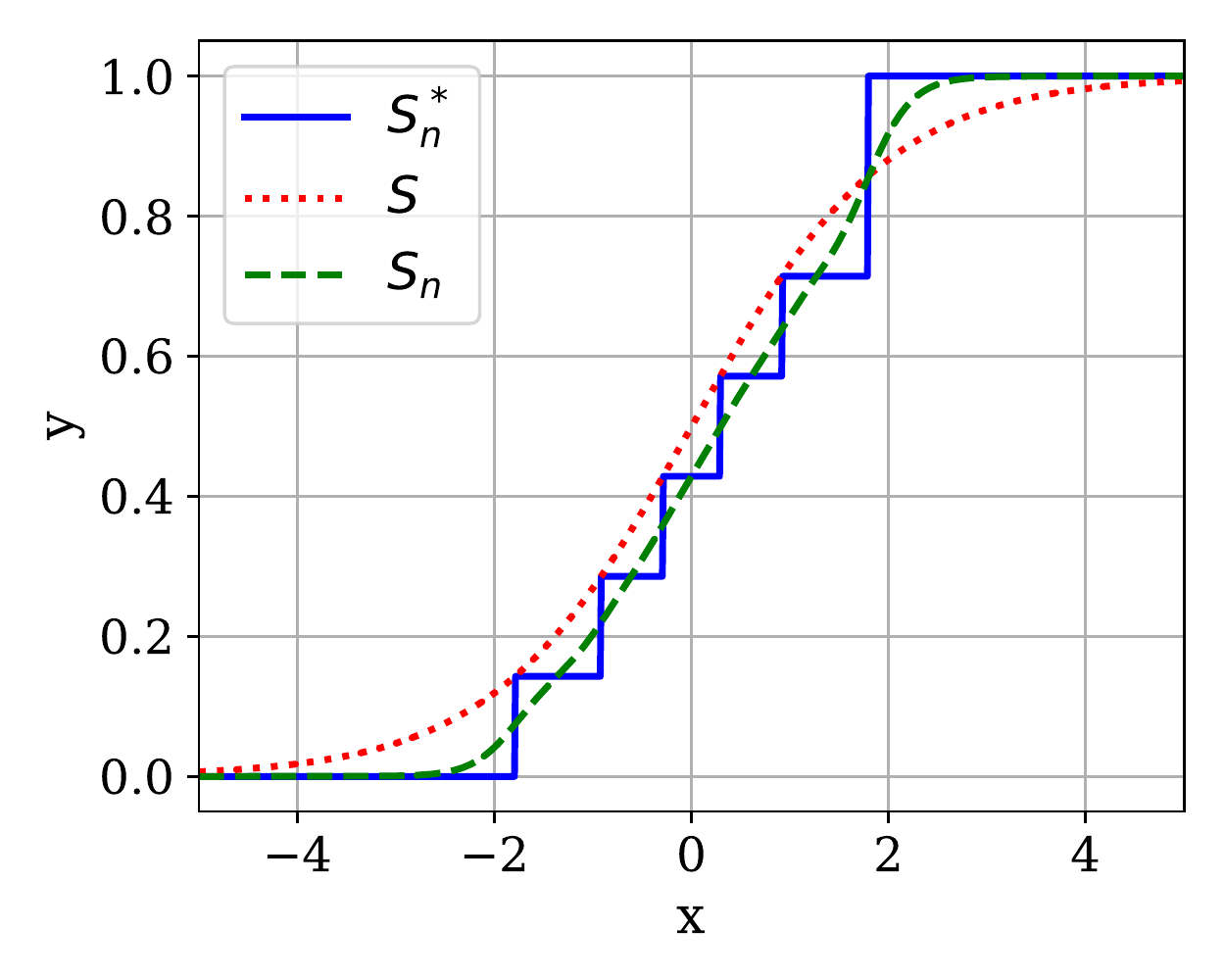}
\caption{Visualization of how sigmoidal functions can universally approximate an monotonic function in $[0,1]$. The red dotted curve is the target monotonic function ($S$), 
and blue solid curve is the intermediate superposition of step functions ($S_n^*$) with the parameters chosen in the proof. 
In this case, $n$ is 6 and $|S_n^{*}-S|\leq \frac{1}{7}$. 
The green dashed curve is the resulting superposition of sigmoids ($S_n$), that intercepts with each step of $S_n^*$ with the additionally chosen $\tau$. 
}
\label{fig:steps}
\end{figure}
We want to show the convergence result of Equation~\ref{eq:prelogit-dsf}.
To this end, we first show that DSF can be used to universally approximate any strictly monotonic function. The is the case where $x_{1:t-1}$ are fixed, which means $\CCC(x_{1:t-1})$ are simply constants.
We demonstrate it using the following two lemmas.

\begin{lemma}
\label{lemma:stepfunc}
\textnormal{(Step functions universally approximate monotonic functions)}
Define:
\begin{align}
S_n^{*}(x) = \sum_{j=1}^{n}w_j\cdot s\left({x-b_j}\right)   \nonumber
\end{align}

where $s(z)$ is defined as a step function that is $1$ when $z\geq0$ and $0$ otherwise.
For any continuous, strictly monotonically increasing $S:[r_0,r_1] \rightarrow[0,1]$ where $S(r_0)=0$, $S(r_1)=1$ and $r_0,r_1\in\R$; 
and given any $\epsilon>0$, 
there exists a positive integer $n$,
real constants $w_j$ and $b_j$ for $j=1,...,n$, where $\sum_j^nw_j=1$, $w_j>0$ and $b_j\in[r_0,r_1]$ for all $j$, such that
\linebreak $|S_n^*(x)-S(x)|<\epsilon \quad\forall x\in[r_0,r_1]$.
\label{prop:sft}
\end{lemma}

\begin{proof} $\textnormal{(Lemma~\ref{lemma:stepfunc})}$

For brevity, we write $s_j(x)=s(x-b_j)$.
For any $\epsilon>0$, we choose $n=\lceil\frac{1}{\epsilon}\rceil$, and divide the range $(0,1)$ into $n+1$ evenly spaced intervals: $(0,y_1)$, $(y_1,y_2)$, $...$, $(y_n,1)$.
For each $y_j$, there is a corresponding inverse value since $S$ is strictly monotonic, $x_j=S^{-1}(y_j)$.
We want to set $S_n^*(x_j)=y_j$ for $1\leq j\leq n-1$ and $S_n^*(x_n)=1$.
To do so, we set the bias terms $b_j$ to be $x_j$.
Then we just need to solve a system of $n$ linear equations $\sum_{j'=1}^nw_{j'}\cdot s_{j'}(x_{j})=t_j$, where $t_j=y_j$ for $1\leq j<n$, $t_0=0$ and $t_n=1$. 
%\textcolor{red}{So $t_j \doteq \frac{j}{n}$?}
We can express this system of equations in the matrix form as $\mathbf{S}\mathbf{w}=\mathbf{t}$, where:
\vspace{-2pt}
\begin{align}
& \mathbf{S}_{j,{j'}}=s_{j'}(x_j)=\delta_{x_j\geq b_{j'}}=\delta_{j\geq {j'}} \,, \nonumber \\  
& \mathbf{w}_j=w_j \,, \quad \mathbf{t}_j=t_j \nonumber
\end{align}
where $\delta_{\omega\geq\eta}=1$ whenever $\omega\geq\eta$ and $\delta_{\omega\geq\eta}=0$ otherwise. 
Then we have $\mathbf{w}=\mathbf{S}^{-1}\mathbf{t}$. 
Note that ${\bf S}$ is a lower triangular matrix, and its inverse takes the form of a Jordan matrix: $(\mathbf{S}^{-1})_{i,i}=1$ and $(\mathbf{S}^{-1})_{i+1,i}=-1$. 
Additionally, $t_j-t_{j-1}=\frac{1}{n+1}$ for $j=1,...,n-1$ and is equal to $\frac{2}{n+2}$ for $j=n$.
We then have $S_n^*(x)=\mathbf{t}^T\mathbf{S}^{-T}\mathbf{s}(x)$, where $\mathbf{s}(x)_j=s_j(x)$; thus
\begin{align}
|S_n^*(x)-S(x)|&=&&|\sum_{j=1}^{n}s_j(x)(t_j-t_{j-1})-S(x)| \nonumber\\
%&&&\textrm{expanding the sum and simplifying yields}\\
%&=&&|s_n(x) + \frac{1}{n+1} \sum_{j=1}^{n-1}js_j(x)  \\
%- \frac{1}{n+1} \left(\sum_{j=1}^{n-1}(j-1)s_j(x) + s_n(x)(n-1) \right) - S(x)| &&& \nonumber\\
&=&&|\frac{1}{n+1}\sum_{j=1}^{n-1}s_j(x) + \frac{2s_n(x)}{n+1} - S(x)| \nonumber\\
&=&& |\frac{C_{y_{1:n-1}}(x)}{n+1}+\frac{2\delta_{x\geq y_n}}{n+1} - S(x) |\nonumber\\
&\leq&&\frac{1}{n+1}<\frac{1}{\lceil\vfrac{1}{\epsilon}\rceil}\leq \epsilon
%&\leq&&| \epsilon \sum_j^{n-1}\delta_{x\geq y_j} + 2\epsilon \delta_{x\geq y_n}
\end{align}
where $C_v(z)=\sum_k\delta_{z\geq v_k}$ is the count of elements in a vector that $z$ is no smaller than. 
\end{proof}

Note that the additional constraint that $\mathbf{w}$ lies on an $n-1$ dimensional simplex is always satisfied, because 
\begin{align}
\sum_jw_j=\sum_{j}t_j-t_{j-1}=t_n-t_0=1 \nonumber
\end{align}
See Figure \ref{fig:steps} for a visual illustration of the proof. 
Using this result, we now turn to the case of using sigmoid functions instead of step functions.

\begin{lemma}
\label{lemma:dsf}
\textnormal{(Superimposed sigmoids universally approximate monotonic functions)}
Define:
\begin{align}
S_n(x) = \sum_{j=1}^{n}w_j\cdot\sigma\left(\frac{x-b_j}{\tau_j}\right)   \nonumber
\end{align}

With the same constraints and definition in Lemma~\ref{lemma:stepfunc}, given any $\epsilon>0$, there exists a positive integer $n$, real constants $w_j$, $\tau_j$ and $b_j$ for $j=1,...,n$, where additionally $\tau_j$ are bounded and positive, such that 
\linebreak $|S_n(x)-S(x)|<\epsilon \quad\forall x\in(r_0,r_1)$.
\label{prop:sft}
\end{lemma}

\begin{proof} $\textnormal{(Lemma~\ref{lemma:dsf})}$

Let $\epsilon_1=\frac{1}{3}\epsilon$ and $\epsilon_2=\frac{2}{3}\epsilon$. 
We know that for this $\epsilon_1$, there exists an $n$ such that $\left|S_n^*-S\right|<\epsilon_1$.

We chose the same $w_j$, $b_j$ for $j=1,...,n$ as the ones used in the proof of Lemma~\ref{lemma:stepfunc}, and let $\tau_1,...,\tau_n$ all be the same value denoted by $\tau$.

Take $\kappa=\min_{j\neq j'} |b_j-b_{j'}|$ and $\tau=\frac{\kappa}{\sigma^{-1}(1-\epsilon_0)}$ for some $\epsilon_0>0$. Take $\Gamma$ to be a lower triangular matrix with values of 0.5 on the diagonal and 1 below the diagonal.
\begin{align*}
&\max_{j=1,...,n} \left| S_n(b_j)-\Gamma_{j}\cdot w \right| \\
&\quad=\max \left| \sum_{j'}w_{j'}\sigma\left(\frac{b_j-b_{j'}}{\tau}\right) - \sum_{j'}w_{j'}\Gamma_{jj'} \right| \\
&\quad=\max \left| \sum_{j'} w_{j'} \left( \sigma\left(\frac{b_j-b_{j'}}{\tau}\right) - \Gamma_{jj'} \right) \right| \\
&\quad<\max \sum_{j'} w_{j'} \epsilon_0 = \epsilon_0
\end{align*}
The inequality is due to 
\begin{align*}
\sigma\left(\frac{b_j-b_{j'}}{\gamma}\right) &= \sigma\left(\frac{b_j-b_{j'}}{\min_{k\neq k'}b_k-b_{k'}}\sigma^{-1}(1-\epsilon_0)\right) \\
&\left\{ \begin{aligned}
&=0.5 && \text{if } j=j' \\
&\geq1-\epsilon_0 && \text{if } j>j' \\
&\leq\epsilon_0 && \text{if } j<j'
\end{aligned} \right.
\end{align*}
Since the product $\Gamma\cdot w$ represents the half step points of $S_n^*$ at $x=b_j$'s, the result above entails $\left|S_n(x)-S_n^*(x)\right|<\epsilon_2=2\epsilon_1$ for all $x$.
To see this, we choose $\epsilon_0=\frac{1}{2(n+1)}$.
Then $S_n$ intercepts with all segments of $S_n^*$ except for the ends.
We choose $\epsilon_2=2\epsilon_1$ since the last step of $S_n^*$ is of size $\frac{2}{n+1}$, and thus the bound also holds true in the vicinity where $S_n$ intercepts with the last step of $S_n^*$.

Hence,
\begin{align*}
&\left| S_n(x)-S(x) \right| \\
&\qquad\leq \left| S_n(x) - S_n^*(x) \right| + \left| S_n^*(x) - S(x) \right|<\epsilon_1+\epsilon_2=\epsilon
\end{align*}

\end{proof}

%Now we show convergence result of Equation~\ref{eq:prelogit-dsf}.
Now we show that (the pre-logit) DSF (Equation~\ref{eq:prelogit-dsf}) can universally approximate monotonic functions.
We do this by showing that the well-known universal function approximation properties of neural networks \citep{cybenko1989approximation} allow us to produce parameters which are sufficiently close to those required by Lemma 2.

\begin{lemma}
\label{lemma:dsf_composed}
%\textnormal{(DSF universally approximates monotonic functions)}
Let $x_{1:m}\in[r_0,r_1]^{m}$ where $r_0,r_1\in\R$.
Given any $\epsilon>0$ and any multivariate continuously differentiable function \footnote{ $S(\cdot): [r_0,r_1]^{m}\rightarrow [0,1]^{m}$ is a multivariate-multivariable function, where $S(\cdot)_t$ is its $t$-th component, which is a univariate-multivariable function, written as $S_t(\cdot,\cdot):[r_0,r_1]\times[r_0,r_1]^{t-1}\rightarrow[0,1]$.} 
$S(x_{1:m})_t=S_t(x_t,x_{1:t-1})$ for $t\in[1,m]$ that is strictly monotonic with respect to the first argument when the second argument is fixed, where the boundary values are $S_t(r_0,x_{1:t-1})=0$ and $S_t(r_1,x_{1:t-1})=1$ for all $x_{1:t-1}$ and $t$, then there exists a multivariate function $\SSS$ such that $\|\SSS(x_{1:m})-S(x_{1:m})\|_{\infty}<\epsilon$ for all $x_{1:m}$, of the following form:
\begin{align}
\SSS(x_{1:m})_t &= \SSS_t(x_t,\CCC_t(x_{1:t-1})) \nonumber \\
&= \sum_{j=1}^{n}w_{tj}(x_{1:t-1})\cdot\sigma\left(\frac{x_t-b_{tj}(x_{1:t-1})}{\tau_{tj}(x_{1:t-1})}\right) \nonumber
\end{align}
where $t \in [1, m]$, and $\mathcal{C}_t=(w_{tj},b_{tj},\tau_{tj})_{j=1}^n$ are functions of $x_{1:1-t}$ parameterized by neural networks, with $\tau_{tj}$ bounded and positive, $b_{tj} \in [r_0,r_1]$, $\sum_{j=1}^nw_{tj}=1$, and $w_{tj}>0$

\end{lemma}

\begin{proof} $\textnormal{(Lemma~\ref{lemma:dsf_composed})}$

First we deal with the univariate case (for any t) and drop the subscript $t$ of the functions. 
We write $\SSS_n$ and $\CCC_k$ to denote the sequences of univariate functions. 
We want to show that for any $\epsilon>0$, there exist (1) a sequence of functions $\SSS_n(x_t,\CCC_k(x_{1:t-1}))$ in the given form, and (2) large enough $N$ and $K$ such that when $n\geq N$ and $k\geq K$, $|\SSS_n(x_t,\CCC_k(x_{1:t-1}))-S(x_t,x_{1:t-1})|\leq\epsilon$ for all $x_{1:t}\in[r_0,r_1]^t$.

The idea is first to show that we can find a sequence of parameters, $\CCC_{n}(x_{1:t-1})$, that yield a good approximation of the target function, $S(x_t,x_{1:t-1})$.
We then show that these parameters can be arbitrarily well approximated by the outputs of a neural network, $\CCC_{k}(x_{1:t-1})$, which in turn yield a good approximation of $S$.

From Lemma \ref{lemma:dsf}, we know that such a sequence $\CCC_{n}(x_{1:t-1})$ exists, and furthermore that we can, for any $\epsilon$, and independently of $S$ and $x_{1:t-1}$ choose an $N$ large enough so that:
\begin{align}
|\SSS_n(x_t,\CCC_{n}(x_{1:t-1})) - S(x_t,x_{1:t-1})|<\frac{\epsilon}{2} \label{eq:l3firstep}
\end{align}

To see that we can further approximate a given $\CCC_{n}(x_{1:t-1})$ well by $\CCC_{k}(x_{1:t-1})$ \footnote{ Note that $\CCC_n$ is a chosen function (we are not assuming its parameterization; i.e. not necessarily a neural network) that we seek to approximate using $\CCC_k$, which is the output of a neural network.}, we apply the classic result of \citet{cybenko1989approximation}, which states that a multilayer perceptron can approximate any continuous function on a compact subset of $\mathbb{R}^m$.
Note that specification of $\CCC_{n}(x_{1:t-1})=(w_{tj},b_{tj},\tau_{tj})_{j=1}^n$ in Lemma \ref{lemma:dsf} depends on the quantiles of $S_t(x_t, \cdot)$ as a function of $x_{1:t-1}$; since the quantiles are continuous functions of $x_{1:t-1}$, so is $\CCC_{n}(x_{1:t-1})$, and the theorem applies.

Now, $\SSS$ has bounded derivative wrt $\CCC$, and is thus uniformly continuous, so long as $\tau$ is greater than some positive constant, which is always the case for any fixed $\CCC_n$, and thus can be guaranteed for $\CCC_k$ as well (for large enough $k$).
Uniform continuity allows us into translate the convergence of $\CCC_k \rightarrow \CCC_n$ to convergence of $\SSS_n(x_t,\CCC_{k}(x_{1:t-1})) \rightarrow \SSS_n(x_t,\CCC_{n}(x_{1:t-1}))$, since for any $\epsilon$, there exists a $\delta > 0$ such that 
\begin{align}
&\| \CCC_k(x_{1:t-1})-\CCC_{n}(x_{1:t-1}) \|_{\infty}<\delta \nonumber \\
&\implies
\left| \SSS_n(x_t,\CCC_k(x_{1:t-1})) - \SSS_n(x_t,\CCC_{n}(x_{1:t-1})) \right|<\frac{\epsilon}{2}
\end{align}

Combining this with Equation \ref{eq:l3firstep}, we have for all $x_t$ and $x_{1:t-1}$, and for all $n\geq N$ and $k\geq K$ 
\begin{align*}
&\left|\SSS_n\left(x_t,\CCC_k(x_{1:t-1})\right)-S(x_t,x_{x_{1:t-1}})\right| \\
&\qquad\quad\leq\left|\SSS_n\left(x_t,\CCC_k(x_{1:t-1})\right) - \SSS_n(x_t,\CCC_{n}(x_{1:t-1}))\right| +  \\
&\qquad\qquad\,\left|\SSS_n(x_t,\CCC_{n}(x_{1:t-1})) - S(x_t,x_{x_{1:t-1}})\right| \\
&\qquad\quad<\frac{\epsilon}{2}+\frac{\epsilon}{2}=\epsilon
\end{align*}

Having proved the univariate case, we add back the subscript $t$ to denote the index of the function. 
From the above, we know that given any $\epsilon>0$ for each $t$, there exist $N(t)$ and a sequence of univariate functions $\SSS_{n,t}$ such that for all $n\geq N(t)$, $|\SSS_{n,t}-S_t|<\epsilon$ for all $x_{1:t}$. 
Choosing $N=\max_{t\in[1,m]} N(t)$, we have that there exists a sequence of multivariate functions $\SSS_n$ in the given form such that for all $n\geq N$, $\|\SSS_n-S\|_{\infty}<\epsilon$ for all $x_{1:m}$. 

\end{proof}

\section{Proof of Universal Approximation of DSF}
\label{sec:main_proof}

\begin{lemma}
\label{lemma:conv_dist}
Let $X\in\mathcal{X}$ be a random variable, and $\mathcal{X}\subseteq\R^m$ and $\mathcal{Y}\subseteq\R^m$. 
Given any function $J:\mathcal{X}\rightarrow\mathcal{Y}$ and a sequences of functions $J_n$ that converges pointwise to $J$, the sequence of random variables induced by the transformations $Y_n\doteq J_n(X)$ converges in distribution to $Y\doteq J(X)$. 
\end{lemma}
\begin{proof} \textnormal{(Lemma 4)}

Let $h$ be any bounded, continuous function on $\R^m$, so that $h\circ J_n$ converges pointwise to $h\circ J$ by continuity of $h$.
Since $h$ is bounded, then by the {\bf dominated convergence theorem}, $\mathbb{E}[h(Y_n)]=\mathbb{E}[h(J_n(X))]$ converges to $\mathbb{E}[h(J(X)]=\mathbb{E}[h(Y)]$.
As this result holds for any bounded continuous function $h$, by the {\bf Portmanteau's lemma}, we have $Y_n\xrightarrow{d}Y$.
\end{proof}

%%%%%%%%%%%%%%%%%%%%%%%%%%%%%%%%%%%%%%%%%%
%%%%%%%%%%%%%%%%%%%%%%%%%%%%%%%%%%%%%%%%%%
%-%  unstructured R.V. to structured R.V.
%%%%%%%%%%%%%%%%%%%%%%%%%%%%%%%%%%%%%%%%%%
%%%%%%%%%%%%%%%%%%%%%%%%%%%%%%%%%%%%%%%%%%
\begin{proof} $\textnormal{(Proposition~\ref{theorem:unstruct2struct})}$

Given an arbitrary ordering, let $F$ be the CDFs of $Y$, defined as $F_t(y_t,x_{1:t-1}) = \Pr(Y_t \leq y_t|x_{1:t-1})$.
According to Theorem 1 of~\citet{hyvarinen1999nonlinear}, $F(Y)$ is uniformly distributed in the cube $[0,1]^m$.
$F$ has an upper triangular Jacobian matrix, whose diagonal entries are conditional densities which are positive by assumption.
Let $G$ be a multivariate and multivariable function where $G_t$ is the inverse of the CDF of $Y_t$: $G_t(F_t(y_t,x_{1:t-1}),x_{1:t-1})=y_t$.

According to Lemma~\ref{lemma:dsf_composed}, there exists a sequence of functions in the given form $(\SSS_n)_{n\geq1}$ that converge uniformly to $\sigma\circ G$. 
Since uniform convergence implies pointwise convergence, $G_n=\sigma^{-1}\circ S_n$ converges pointwise to $G$, by continuity of $\sigma^{-1}$. 
Since $G_n$ converges pointwise to $G$ and $G(X)=Y$, by Lemma~\ref{lemma:conv_dist}, we have $Y_n\xrightarrow{d}Y$

\end{proof}
\vspace{-20pt}
%%%%%%%%%%%%%%%%%%%%%%%%%%%%%%%%%%%%%%%%%%
%%%%%%%%%%%%%%%%%%%%%%%%%%%%%%%%%%%%%%%%%%
%-%  structured R.V. to unstructured R.V.
%%%%%%%%%%%%%%%%%%%%%%%%%%%%%%%%%%%%%%%%%%
%%%%%%%%%%%%%%%%%%%%%%%%%%%%%%%%%%%%%%%%%%
\begin{proof}
$\textnormal{(Proposition~\ref{theorem:struct2unstruct})}$

Given an arbitrary ordering, let $H$ be the CDFs of $X$:
\begin{align*}
y_1 &\doteq H_1(x_1,\emptyset) =F_1(x_1,\emptyset)=\textnormal{Pr}(X_1\leq x_1|\emptyset)\\
y_t &\doteq H_t(x_t, x_{1:t-1}) \\
&=F_t\big(x_t,\{H_{t-t'}(x_{t-t'},x_{1:t-t'-1})\}_{t'=1}^{t-1}\big)\\
%&=F_t(x_t,\cap_{t'=1}^{t-1} H_{t-t'}(x_{t-t'},x_{1:t-t'-1}))\\
%&=F_t(x_t,\big\{H_{1}(x_{1},\emptyset),...,H_{t-1}(x_{t-1},x_{1:t-2})\big\})\\
&=\textnormal{Pr}(X_t\leq x_t|y_{1:{t-1}}) \quad\textnormal{ for }2\leq t\leq m
\end{align*}
Due to~\citet{hyvarinen1999nonlinear}, $y_1,...y_m$ are independently and uniformly distributed in $(0,1)^m$. 

According to Lemma~\ref{lemma:dsf_composed}, there exists a sequence of functions in the given form $(\SSS_n)_{n\geq1}$ that converge uniformly to $H$. 
Since $H_n=\SSS_n$ converges pointwise to $H$ and $H(X)=Y$, by Lemma~\ref{lemma:conv_dist}, we have $Y_n\xrightarrow{d}Y$

\end{proof}
\vspace{-20pt}
%%%%%%%%%%%%%%%%%%%%%%%%%%%%%%%%%%%%%%%%%%
%%%%%%%%%%%%%%%%%%%%%%%%%%%%%%%%%%%%%%%%%%
%-%  structured R.V. to structured R.V.
%%%%%%%%%%%%%%%%%%%%%%%%%%%%%%%%%%%%%%%%%%
%%%%%%%%%%%%%%%%%%%%%%%%%%%%%%%%%%%%%%%%%%
\begin{proof}
$\textnormal{(Theorem~\ref{theorem:struct2struct})}$

Given an arbitrary ordering, let $H$ be the CDFs of $X$ defined the same way in the proof for Proposition~\ref{theorem:struct2unstruct}, and let $G$ be the inverse of the CDFs of $Y$ defined the same way in the proof for Proposition~\ref{theorem:unstruct2struct}. 
Due to \citet{hyvarinen1999nonlinear}, $H(X)$ is uniformly distributed in $(0,1)^m$, so $G(H(X))=Y$.
Since $H_t(x_t, x_{1:t-1})$ is monotonic wrt $x_t$ given $x_{1:t-1}$, and $G_t(H_t, H_{1:t-1})$ is monotonic wrt $H_t$ given $H_{1:t-1}$, $G_t$ is also monotonic wrt $x_t$ given $x_{1:t-1}$, as 
$$ \frac{\partial G_t(H_t,H_{1:t-1})}{\partial x_t} =
\frac{\partial G_t(H_t,H_{1:t-1})}{\partial H_t} 
\frac{\partial H_t(x_t,x_{1:t-1})}{\partial x_t} $$
is always positive. 

According to Lemma~\ref{lemma:dsf_composed}, there exists a sequence of functions in the given form $(\SSS_n)_{n\geq1}$ that converge uniformly to $\sigma\circ G\circ H$. 
Since uniform convergence implies pointwise convergence, $K_n=\sigma^{-1}\circ S_n$ converges pointwise to $G\circ H$, by continuity of $\sigma^{-1}$. 
Since $K_n$ converges pointwise to $G\circ H$ and $G(H(X))=Y$, by Lemma~\ref{lemma:conv_dist}, we have $Y_n\xrightarrow{d}Y$

\end{proof}
\vspace{-20pt}
%%%%%%%%%%%%%%%%%%%%%%%%%%%%%%%%%%%%%%%%%%
%%%%%%%%%%%%%%%%%%%%%%%%%%%%%%%%%%%%%%%%%%
%-%  Experimental Details
%%%%%%%%%%%%%%%%%%%%%%%%%%%%%%%%%%%%%%%%%%
%%%%%%%%%%%%%%%%%%%%%%%%%%%%%%%%%%%%%%%%%%
\section{Experimental Details}
For the experiment of amortized variational inference, we implement the Variational Autoencoder~\cite{kingma2013auto}.
Specifically, we follow the architecture used in \citet{kingma2016improved}: the encoder has three layers with $[16,32,32]$ feature maps. 
We use resnet blocks \cite{he2016deep} with $3\times3$ convolution filters and a stride of $2$ to downsize the feature maps. 
The convolution layers are followed by a fully connected layer of size $450$ as a context for the flow layers that transform the noise sampled from a standard normal distribution of dimension $32$.
The decoder is symmetrical with the encoder, with the strided convolution replaced by a combination of bilinear upsampling and regular resnet convolution to double the feature map size. 
We used the ELUs activation function~\cite{clevert2015fast} and weight normalization~\cite{salimans2016weight} in the encoder and decoder.
In terms of optimization, Adam~\cite{Kingma2015} is used with learning rate fined tuned for each inference setting, and Polyak averaging~\cite{polyak1992acceleration} was used for evaluation with $\alpha=0.998$ which stands for the proportion of the past at each time step. 
We also consider a variant of Adam known as Amsgrad~\cite{reddi2018convergence} as a hyperparameter. 
For vanilla VAE, we simply apply a resnet dot product with the context vector to output the mean and the pre-softplus standard deviation, and transform each dimension of the noise vector independently. 
We call this linear flow. 
For IAF-affine and IAF-DSF, we employ MADE~\cite{MADE} as the conditioner $c(x_{1:t-1})$, and we apply dot product on the context vector to output a scale vector and a bias vector to conditionally rescale and shift the preactivation of each layer of the MADE. 
Each MADE has one hidden layer with $1920$ hidden units. 
The IAF experiments all start with a linear flow layer followed by IAF-affine or IAF-DSF transformations. 
For DSF, we choose $d=16$.

For the experiment of density estimation with MAF, we followed the implementation of~\citet{papamakarios2017masked}.
Specifically for each dataset, we experimented with both $5$ and $10$ flow layers, followed by one linear flow layer. 
The following table specifies the number of hidden layers and the number of hidden units per hidden layer for MADE:
\begin{table}[h]
\caption{
Architecture specification of MADE in the MAF experiment. Number of hidden layers and number of hidden units.
}
\label{tb:maf_spec}
\centering
\begin{tabular}{ccccc}
\toprule
{\scriptsize POWER} & 
{\scriptsize GAS} & 
{\scriptsize HEPMASS} & 
{\scriptsize MINIBOONE} & 
{\scriptsize BSDS300} \\
\midrule
${2\times100}$ & 
${2\times100}$ &
${2\times512}$ &
${1\times512}$ &
${2\times1024}$
\\
\bottomrule
\end{tabular}
\end{table}

\end{document}